\documentclass[runningheads]{llncs}

\usepackage{misc}
\usepackage[T1]{fontenc}
\usepackage{soul}
\usepackage{url}
\usepackage[hidelinks]{hyperref}
\usepackage[utf8]{inputenc}
\usepackage{graphicx}
\usepackage{amsmath}
\usepackage{booktabs}
\usepackage{xspace}

\usepackage{amssymb}
\usepackage{float}
\usepackage{stmaryrd}
\usepackage{enumitem}

\usepackage{amsfonts}
\usepackage{wasysym}
\usepackage[noend,linesnumbered]{algorithm2e}
\usepackage{cite}
\usepackage[small, labelfont = bf]{caption}
\usepackage{subcaption}
\usepackage{tabularx}
\usepackage{multirow, makecell}

\usepackage{tikz}
\usetikzlibrary{positioning,chains,arrows,arrows.meta,shadows,fadings,shapes,backgrounds,matrix,patterns,plotmarks,trees,mindmap,automata,fit,calc,decorations.text}
\usepackage{pgfplots, pgfplotstable}
\usepackage{mdframed}
\usepackage{thmtools}
\usepackage{thm-restate}

\usepackage{microtype}

\newcommand{\alc}{\mathcal{ALC}}

\newcommand{\oall}{\ensuremath{O_{\text{all}}}\xspace}

\allowdisplaybreaks

\newcommand*{\decproblem}[3]{
    \begin{center}
        \begin{tabularx}{0.9\linewidth}{rX}
            \multicolumn{2}{l}{\textbf{Problem:} \textsc{#1}}\\\hline
            \textsl{Given:}&#2\\
            \textsl{Question:}&#3
        \end{tabularx}
    \end{center}
}

\begin{document}

\title{SAT-Based Bounded Fitting for the Description~Logic \ALC}

\author{Maurice Funk\inst{1}\orcidID{0000-0003-1823-9370} \and
Jean Christoph Jung\inst{2}\orcidID{0000-0002-4159-2255} \and
Tom Voellmer\inst{2}\orcidID{0009-0006-3149-7066}}
\authorrunning{M. Funk et al.}
%
\institute{Leipzig University and ScaDS.AI, \email{mfunk@informatik.uni-leipzig.de} \and 
TU Dortmund University, \email{\{jean.jung, tom.voellmer\}@tu-dortmund.de}}

\maketitle

\begin{abstract}
Bounded fitting is a general paradigm for learning logical formulas from positive and negative data examples, that has received considerable interest recently. We investigate bounded fitting for the description logic $\ALC$ and its syntactic fragments. We show that the underlying size-restricted fitting problem is $\NPclass$-complete for all studied fragments, even in the special case of a single positive and a single negative example. By design, bounded fitting comes with probabilistic guarantees in Valiant's PAC learning framework. In contrast, we show that other classes of algorithms for learning $\ALC$ concepts do not provide such guarantees. Finally, we present an implementation of bounded fitting in \ALC and its fragments based on a SAT solver. We discuss optimizations and compare our implementation to other concept learning tools.

\end{abstract}

\section{Introduction}
Learning \emph{class expressions} from given data examples is an important task when working with large knowledge bases~\cite{DBLP:series/ssw/Lehmann10,polconcept_learning}.
For the purpose of this paper, an \emph{example} is a pair $(\Imc,a)$ where $\Imc$ is a data instance, e.g., a database or a knowledge graph, and $a$ is some element in the instance. Moreover, a class expression $C$ \emph{fits} a set $P$ of \emph{positive examples} and a set $N$ of \emph{negative examples} if $\Imc\models C(a)$ for all $(\Imc,a)\in P$ and $\Imc\not\models C(a)$ for all $(\Imc,a)\in N$. We mention three applications. First, the fitting class expression may be used as an \emph{explanation} of the separation between good and bad scenarios, described by $P$ and $N$, respectively. For example, $P$ and $N$ could be data describing users who visited (resp., did not visit) a certain web page, and a fitting $C$ would explain the users' behavior based on their data. Second, under the classical \emph{query-by-example} paradigm~\cite{DBLP:conf/afips/Zloof75,DBLP:journals/is/Martins19},
a human user may \emph{reverse-engineer} a class expression by manually selecting elements they want to have returned ($P$) or not returned ($N$), and a system produces an expression satisfying the demands. Here, the intended use of the fitting class expression is as a query. Finally, an ontology engineer may seek a definition of some class $A$ occurring in the data instance, so they may request a class expression $C$ separating all elements $P$ satisfying $A$ from those that do not ($N$), and add $A\equiv C$ to the ontology.

In this paper, we study the problem of learning class expressions formulated in the description logic (DL) \ALC. $\ALC$ can be viewed as the conceptual basis of the very expressive description logic $\SROIQ$ that underlies the web ontology language OWL 2 DL~\cite{DBLP:conf/kr/HorrocksKS06,OWL2DLweb} in that $\ALC$ provides the central logical constructors that are also available in $\SROIQ$. The importance of finding fitting class expressions / description logic concepts has resulted in both foundational work~\cite{DBLP:conf/ijcai/FunkJLPW19,DBLP:journals/ml/LehmannH10,FJL-IJCAI21} and implemented systems. While most systems are based on 
heuristic search and refinement operators~\cite{DBLP:conf/www/HeindorfBDWGDN22,DBLP:journals/jmlr/TranDGM17,DBLP:journals/fgcs/RizzoFd20,DBLP:conf/www/BuhmannLWB18,DBLP:journals/apin/IannonePF07} or, more recently, also on neural techniques~\cite{DBLP:conf/ijcai/DemirN23,DBLP:conf/esws/KouagouHDN23}, we approach the problem via bounded fitting. \emph{Bounded fitting} is a general paradigm for fitting logical formulas to
positive and negative examples that has been investigated recently for a range of logics like linear temporal logic LTL~\cite{DBLP:conf/fmcad/NeiderG18,DBLP:conf/aips/CamachoM19},
computational tree logic CTL~\cite{DBLP:conf/ijcar/PommelletSS24}, and the description
logic \EL~\cite{DBLP:conf/ijcai/CateFJL23}. Algorithm~\ref{alg:boundedfitting} provides an abstract description of bounded fitting for a given logic $\Lmc$.
It should be clear that, if any fitting formula exists, bounded fitting always returns a fitting formula of minimal size, which is often a desirable property. From a practical perspective, human users typically prefer shorter, that is, simpler formulas in the applications described above. For instance, when class expressions are used as explanations, a shorter class expression might provide a more transparent explanation. From a theoretical perspective, this property makes bounded fitting an \emph{Occam algorithm}, which implies that it comes with probabilistic generalization guarantees in Valiant's probably approximately correct (PAC) learning framework~\cite{DBLP:journals/cacm/Valiant84,DBLP:journals/jacm/BlumerEHW89}. Intuitively, this means that bounded fitting needs only few examples to be able to generalize to unseen examples. A further advantage of bounded fitting is that it is \emph{complete}: whenever there is a fitting concept bounded fitting will find one.

\begin{algorithm}[t]
\caption{Bounded Fitting for abstract logic \Lmc.}\label{alg:boundedfitting}
   \KwIn{Positive examples $P$, negative examples $N$}
   \For {$k:=1,2,\ldots$}{
        \If{there is $\varphi\in \Lmc$ of size $k$ that fits $P,N$}{\Return $\varphi$}
   }
\end{algorithm}

The basic DL \ALC provides the logical constructors conjunction~$\sqcap$, disjunction~$\sqcup$, negation $\neg$, existential restriction $\exists r$, and universal restriction $\forall r$ to build complex concepts. Motivated by the fact that, depending on the application, one may not need all concept constructors, fragments $\Lmc(O)$ of \ALC have been studied which allow only a subset $O\subseteq \{\sqcap,\sqcup,\neg,\exists,\forall\}$ of the available constructors. 
For instance, the mentioned DL \EL is defined by $\{\sqcap,\exists\}$ and the DL $\FLz$ is defined by $\{\sqcap,\forall\}$. Bounded fitting has been studied recently for \EL, and we extend this study here to all other syntactical fragments.

Our main contributions are as follows. First, we study the \emph{size-restricted fitting problem}: given positive examples $P$, negative examples $N$, and a size bound $k$, determine whether there is a concept of size at most $k$ that fits $P$ and $N$. Clearly, this is precisely the problem to be solved in Line~2 of Algorithm~\ref{alg:boundedfitting}. Our main result is that size-restricted fitting for $\Lmc(O)$ is \NPclass-complete for all $O$ that include at least $\exists$ or $\forall$. While the \NPclass upper bound is straightforward, the lower bound is more technical and rather strong. First, surprisingly, it already applies for inputs consisting of one positive and one negative example, and thus strengthens known hardness results for \EL~\cite{IJCAI23arxiv, DBLP:journals/ipl/CateFJL24}. Second, it uses only a fixed set of symbols in contrast to related hardness results for temporal logics~\cite{DBLP:conf/stacs/BordaisN025,DBLP:journals/corr/abs-2312-16336}.

Then, motivated by the ability of bounded fitting to generalize well from few examples, we discuss the generalization abilities of fitting algorithms for languages $\Lmc(O)$ in Valiant's PAC learning framework. We start with observing that 
under reasonable complexity theoretic assumptions, \emph{no} $\Lmc(O)$ admits an \emph{efficient} PAC learning algorithm. We then analyze the generalization ability of fitting algorithms that have favorable properties from a logical perspective in that they return fitting formulas that are most specific, most general, or of minimal quantifier depth among all fitting formulas. We show that, with one exception, all such algorithms are not \emph{sample-efficient}, and hence do not generalize well. Similar results have been obtained for $\EL$ concepts~\cite{DBLP:conf/ijcai/CateFJL23} and also for the related query language of conjunctive queries~\cite{DBLP:journals/sigmod/CateFJL23}.

Finally, we provide an implementation of bounded fitting for \ALC and its fragments, that relies on a SAT solver to solve size-restricted fitting in Line~2 of bounded fitting. We describe our implementation and present several improvements in the encoding which lead to significant speed-ups. We also present an \emph{approximation scheme}, which is an adaptation of the basic bounded fitting algorithm depicted in Algorithm~\ref{alg:boundedfitting} to be able to provide approximate solutions in case no perfect fitting was found.
We then compare our implementation to other concept learning tools from the literature, both in terms of achieved accuracy, size of the returned concepts, and generalization ability.

\section{Preliminaries}\label{sec:prelims}

We introduce syntax and semantics of the description
logic~\ALC~\cite{DL-Textbook}. Let \NC and \NR be mutually disjoint and countably infinite sets
of \emph{concept names} and \emph{role names}, respectively.
An \emph{\ALC concept} $C$ is defined according to the
syntax rule
\[
C, D ::= \bot\mid \top \mid A \mid \neg C \mid C \sqcap D \mid C\sqcup D\mid \forall r.C \mid \exists r.C
\]
where $A$ ranges over concept names and $r$ over role names. 
For any set $ O\subseteq\{\neg,\sqcap,\sqcup,\forall,\exists\} = \oall $, we define $\mathcal L(O)$ as the set of all $\alc$ concepts built only from connectives and quantifiers from $O$. We associate with each such set $O$ a \emph{dual} set $\overline{O}$ that contains precisely the symbols dual to those in $O$ where $\neg$ is dual to itself, $\sqcap$ is dual to $\sqcup$, and $\exists$ is dual to $\forall$. The size $\lVert C\rVert$ of an $\alc$ concept $C$ is defined as the number of nodes in its syntax tree, see below for an example. A \emph{signature} $\Sigma$ is a finite set of concept and role names. 
The signature \emph{of an $\ALC$ concept} $C$ is the set of concept and role names that occur in $C$.

The semantics of \ALC concepts is given as usual in terms of interpretations. An \emph{interpretation} is a pair $\Imc=(\Delta^\Imc,\cdot^\Imc)$ where $\Delta^\Imc$ is a non-empty set of elements, called \emph{domain}, and $\cdot^\Imc$ is an interpretation function that assigns a set $A^\Imc\subseteq \Delta^\Imc$ to every $A\in \NC$ and a binary relation $r^\Imc\subseteq \Delta^\Imc\times\Delta^\Imc$ to every $r\in \NR$. The \emph{extension} $C^\Imc$ of complex \ALC concepts $C$ is defined inductively by taking $\bot^\Imc=\emptyset$, $\top^\Imc=\Delta^\Imc$, $(\neg C)^\Imc=\Delta^\Imc\setminus C^\Imc$, $(C\sqcap D)^\Imc = C^\Imc\cap D^\Imc$, $(C\sqcup D)^\Imc = C^\Imc\cup D^\Imc$, and
\begin{align*}
    (\exists r.C)^\Imc & = \{d\in\Delta^\Imc\mid \text{there is $(d,e)\in r^\Imc$ with $e\in C^\Imc$}\} 
\end{align*}
and $(\forall r.C)^\Imc = (\neg \exists r.\neg C)^\Imc$.

We call an interpretation $\mathcal I$ \emph{finite}, if both $\Delta^{\mathcal I}$ and the set of all concept and role names that have a nonempty extension in $\mathcal I$ are finite. A finite interpretation $\mathcal I$ corresponds to a finite set of facts
\[\{A(a)\mid A\in\NC, a\in A^\Imc \}\cup\{r(a,b)\mid r\in\NR, (a,b)\in r^\Imc \},\]
and can hence be viewed as a database. Conversely, every database or finite set of facts corresponds to a finite interpretation. We define an example as a pair $(\mathcal I,a)$ consisting of a finite interpretation $\mathcal I$ and some domain element $a\in\Delta^{\mathcal I}$. Let $P$ and $N$ be sets of examples which we will refer to as \emph{positive} and \emph{negative} examples. We say that an \ALC concept $C$ fits $P,N$ if $a\in C^{\mathcal I}$ for each $(\mathcal I,a)\in C$ and $b\notin C^{\mathcal J}$ for each $(\mathcal J, b)\in N$. In the special case of $P=\{(\mathcal I,a_0)\}$ and $N=\{(\mathcal J, b_0)\}$ we just write that $C$ fits $(\mathcal I,a_0),(\mathcal J,b_0)$.

The \emph{size} of such an example $(\Imc,a)$, denoted $\lVert(\Imc,a)\rVert$, is $n+1$ where $n$ is the cardinality of $\Imc$, viewed as set of facts. As usual, we view interpretations as node and edge labeled graphs and use graph terminology to speak about them.

We investigate the following decision problem.
\decproblem{Size-RestrictedFitting for $\mathcal L$}{Sets $P,N$ of examples, integer $k$}{Is there an $\mathcal{L}$ concept $C$ that fits $P,N$ and has size $\lVert C\rVert\le k$?}
We study size-restricted fitting for syntactic fragments $\Lmc(O)$ of \ALC, 
assuming that $k$ is given in unary, which is a natural choice in the context of bounded fitting. 
If we drop the size restriction $k$ from the input, we obtain the \emph{fitting problem for \Lmc}.

\begin{example}
As an example for size-restricted fitting, consider the two positive examples $P=\{(\Imc,a_1),(\Imc,a_2)\}$ and one negative example $N=\{(\Jmc,b)\}$ depicted in Figure~\ref{fig:ex_size_restricted_fitting}. Then $P,N,k=4$ is a yes-instance of size-restricted fitting for \ALC with witnessing fitting concept $C=\forall r.(A\sqcup B)$ with $\lVert \forall r.(A\sqcup B)\rVert=4$, see the syntax tree depicted in the right of Figure~\ref{fig:ex_size_restricted_fitting}. Indeed, we have $a_1,a_2\in C^\Imc$ but $b\notin C^\Jmc$. It can also be verified, that it is the minimal $k$ for which there exists a fitting concept, and that there is no fitting $\Lmc(\{\exists,\sqcap\})$ concept at all, hence $P,N$ is a no-instance for the fitting problem for $\Lmc(\{\exists,\sqcap\})$.\hfill$\dashv$ 
\end{example}

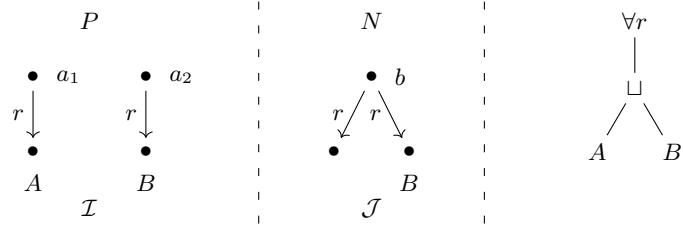
\begin{figure}[t]
    \centering
    \begin{tikzpicture}
            \node (P) at (0.75,0.75) {$P$};
            \node (N) at (4.5,0.75) {$N$};
            
            \node (P) at (0.75,-1.75) {$\Imc$};
            \node (N) at (4.5,-1.75) {$\Jmc$};
     
            \node[label={right:$a_1$}] (11) at (0,0) {$\bullet$};
            \node[label = {below:$A$}] (12) at (0,-1) {$\bullet$};
    
            \node[label={right:$a_2$}] (21) at (1.5,0) {$\bullet$};
            \node[label = {below:$B$}] (22) at (1.5,-1) {$\bullet$};
    
            \draw[loosely dashed] (3,1) -- (3,-2);
    
            \node[label={right:$b$}] (31) at (4.5,0) {$\bullet$};
            \node (32) at (4,-1) {$\bullet$};
            \node[label = {below:$B$}] (33) at (5,-1) {$\bullet$};

            \begin{scope}[every edge/.style={draw, ->}]
              \draw (11) edge node[left] {$r$} (12)
                    (21) edge node[left] {$r$} (22)
                    (31) edge node[left] {$r$} (32)
                    (31) edge node[left] {$r$} (33)
              ;
            \end{scope}
            
            \draw[loosely dashed] (6,1) -- (6,-2);
            
            \node (x1) at (8,0.75) {$\forall r$};
            \node (x2) at (8,-0.15) {$\sqcup$};
            \node (x3) at (7.5,-1) {$A$};
            \node (x4) at (8.5,-1) {$B$};
            
            \begin{scope}[every edge/.style={draw, --}]
              \draw (x1) -- (x2)
                    (x2) -- (x3)
                    (x2) -- (x4)
              ;
            \end{scope}
             
          \end{tikzpicture}        
    \caption{Example for size-restricted fitting}
    \label{fig:ex_size_restricted_fitting}
\end{figure}

To transfer results between dual fragments of \ALC, we rely on an auxiliary lemma that relates fitting in $\Lmc(O)$ to fitting in $\Lmc(\overline O)$. To formulate the lemma, we 
define a \emph{duality} mapping $\overline{\,\cdot\,}$ for \ALC concepts, interpretations and sets of examples. 
For an \ALC concept $C$, the dual concept $\overline C$ is obtained from $C$ by replacing $\top$ with $\bot$ and vice versa, replacing $\sqcap$ with $\sqcup$ and vice versa, and replacing $\exists$ with $\forall$ and vice versa. 
Moreover, for an interpretation $\Imc$ and a signature $\Sigma$, the interpretation $\overline{\Imc}_\Sigma$ is obtained by complementing the interpretation of the concept names in $\Sigma$, that is, $\Delta^{\overline \Imc_\Sigma}=\Delta^\Imc$, 
\[A^{\overline\Imc_\Sigma}=\begin{cases} \Delta^\Imc\setminus A^\Imc & \text{ if $A\in\NC\cap \Sigma$,}\\ A^{\Imc} & \text{ if $A\in\NC\setminus \Sigma$}, \end{cases}\]
and keeping the interpretation of the role names, that is, $r^{\overline \Imc_\Sigma}=r^\Imc$, for all $r\in \NR$.  One can then show inductively that the following duality is satisfied
for all signatures $\Sigma$, $\ALC$ concepts $C$ with signature contained in $\Sigma$, interpretations $\Imc$, and $a\in \Delta^\Imc$: 
\begin{equation}\label{eq:duality}
 a\in C^{\Imc}\quad \text{if and only if}\quad a\notin \overline C^{\overline \Imc_\Sigma}.
\end{equation}
For a set of examples $\mathit{Ex}$ and signature $\Sigma$, we define the set $\overline{\mathit{Ex}}_\Sigma = \{ (\overline \Imc_\Sigma,a)\mid (\Imc,a)\in \mathit{Ex}\}$. 
Duality~\eqref{eq:duality} is used to relate fitting in dual languages as follows:
\begin{restatable}{lemma}{fittingdual}
\label{lem:fitting_dual}
  For any signature $\Sigma$, all sets of examples $P$ and $N$ and $\alc$ concepts $C$ with signature contained in $\Sigma$, $C$ fits $P,N$ if and only if $ \overline{C}$ fits $\overline N_\Sigma,\overline P_\Sigma$.
\end{restatable}
Since for given sets of examples $P,N$ there is a fitting \Lmc concept of size $k$ if and only if there is a fitting \Lmc concept of size $k$ which uses only concept and role names that occur in $P,N$, Lemma~\ref{lem:fitting_dual} implies the following:
\begin{corollary}\label{cor:reduction}
For every $O\subseteq \oall$, there is a polynomial time reduction of size-restricted fitting for $\Lmc(O)$ to size-restricted fitting for $\Lmc(\overline O)$.
\end{corollary}

\section{Complexity of Size-Restricted Fitting}\label{sec:complexity}

We determine the complexity of size-restricted fitting as the decision problem underlying bounded fitting in all syntactic fragments $\mathcal L(O)$ where $O\subseteq \oall$ and $O\cap\{\forall,\exists\}\neq\emptyset$. Our main result states that the problem is \NPclass-complete, for any such $O$. While the proof of the upper bound is straightforward, the \NPclass-hardness proof is rather involved. Notably, it holds already for single positive and negative examples, which came as a surprise to us.

\begin{theorem}\label{thm:main1}
  For every $O\subseteq \oall$ such that $O$ contains at least one of $\exists$ and $\forall$, size-restricted fitting for $\Lmc(O)$ is \NPclass-complete.   
\end{theorem}

Containment in $\NPclass$ follows from a standard guess-and-check argument. Given $P,N$ and $k$, the algorithm non-deterministically picks an $\Lmc(O)$ concept $C$ of size at most $k$ and 
verifies (in polynomial time) that $a\in C^{\mathcal I}$ for all $(\mathcal I,a)\in P$ and $b\notin C^{\mathcal J}$ for all $(\mathcal J,b)\in N$. 
Showing hardness is more involved and proceeds in two steps. First, with the following proposition, we show \NPclass-hardness for any $O$ with $\exists \in O$.
Then, Corollary~\ref{cor:reduction} 
implies \NPclass-hardness for any $O$ with $\forall \in O$.
    
    \begin{restatable}{proposition}{proplower}
        \label{prop:technplower}
        Size-restricted fitting for $\Lmc(O)$ is \NPclass-hard for any $O\subseteq \oall$ with $\exists \in O$. This already holds if only a single positive and a single negative example are allowed, and over a signature consisting of two role names and one concept name.
    \end{restatable}
    
\begin{proof}
  We give a reduction of \emph{Hitting Set} which is the problem to decide, given a  collection of sets $S=\{S_1,\dots,S_m\}$ and a size bound $k\in\mathbb N$, whether there exists a set $H$ such that $|H|\le k$ and $H\cap S_j\ne\emptyset$, for each $j\in\{1,\dots,m\}$.  The set $H$ is called a hitting set for $S$. 

Let $(S,k)$ be an instance of \emph{Hitting Set} with $S=\{S_1,\dots,S_m\}$ a collection of sets $S_i\subseteq\mathbb{N}$ and $k\in\mathbb N$. Assume without loss of generality that $\bigcup_{j=1}^m S_j = \{1,\dots,n\}$ for some $n\in\mathbb N$. We construct a positive example $(\mathcal I,a)$ and a negative example $(\mathcal J,b)$ in which we interpret only two role names $r,s$ and a concept name $A$ non-empty, and such that, for $P=\{(\Imc,a)\}$, $N=\{(\Jmc,b)\}$, and $k'=k+n+2$,
\begin{equation}
S\text{ has hitting set of size $k$}\quad\Leftrightarrow\quad\text{$P,N$ admit an $\Lmc(O)$ fitting of size $k'$}.\label{eq:correctness}
\end{equation}
We provide an informal description of $\Imc,\Jmc$; more details can be found in the appendix. The main components of both $\Imc$ and \Jmc are $r$-paths of length $n$, which end in an element satisfying $A$. To encode a subset $S'\subseteq \{1,\ldots,n\}$, such a path is extended with additional ``detours'' along role $s$. In more detail, suppose such a path consists of elements $b_0,\ldots,b_n$. We encode a subset $S'\subseteq \{1,\ldots,n\}$ as follows: 
\begin{itemize}

    \item For all $i\in\{1,\ldots,n\}\setminus S'$, there is an $s$-path of length $2$ from $b_i$ to $b_{i+1}$.

    \item For every $b_i$ and role $t\in\{r,s\}$: if $b_i$ does not have a $t$-successor, then make $c$ a $t$-successor of $b_i$ where $c$ is a \emph{sink}, which has both an $r$ and an $s$ self loop.
    
\end{itemize}
Now, $\Jmc$ consists of one path for each $S_j\in S$ and an additional root node $b$ which connects via $r$ to the initial nodes of all the paths. The interpretation \Imc is an extension of \Jmc with an additional component that encodes the empty set, and has root node $a$ which connects via $r$ to the initial nodes of all paths.  Figure~\ref{fig:np} depicts the constructed interpretations for the instance $(S,2)$ for $S=\{ \{1,3\}, \{2,4\}\}$. We omit the sink element $c$ in the figure for the sake of readability. For this input, $H=\{1,2\}$ is a hitting set of size $2$.
        
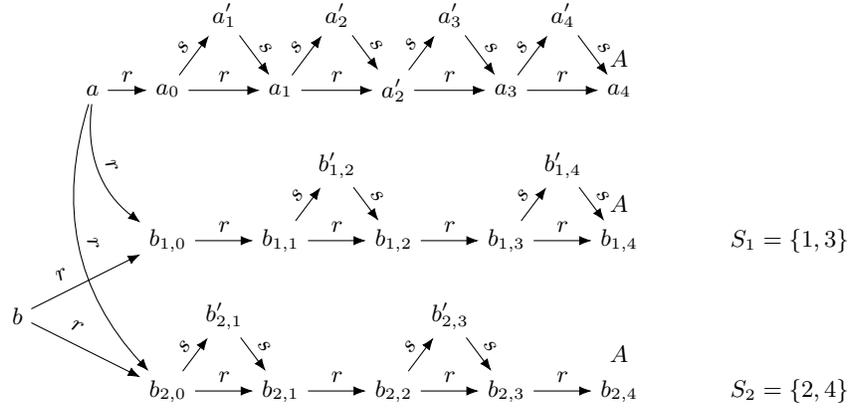
\begin{figure}[t]
            \begin{center}
            \begin{tikzpicture}
            \node[label = ${ }$] (a) at (-1,-2.0) {$a$};
            \node[label = ${ }$] (a0) at (0.0,-2.0) {$a_{0}$};
            \node[label = ${ }$] (a1) at (1.5,-2.0) {$a_{1}$};
            \node[label = ${ }$] (ap1) at (0.75,-1.0) {$a_{1}'$};
            \node[label = ${ }$] (a2) at (3.0,-2.0) {$a_{2}'$};
            \node[label = ${ }$] (ap2) at (2.25,-1.0) {$a_{2}'$};
            \node[label = ${ }$] (a3) at (4.5,-2.0) {$a_{3}$};
            \node[label = ${ }$] (ap3) at (3.75,-1.0) {$a_{3}'$};
            \node[label = ${A }$] (a4) at (6.0,-2.0) {$a_{4}$};
            \node[label = ${ }$] (ap4) at (5.25,-1.0) {$a_{4}'$};

            \draw[-Latex] (a0) edge node[above, sloped] {$ r$} (a1);
            \draw[-Latex] (a0) edge node[above, sloped] {$ s$} (ap1);
            \draw[-Latex] (a1) edge node[above, sloped] {$ r$} (a2);
            \draw[-Latex] (a1) edge node[above, sloped] {$ s$} (ap2);
            \draw[-Latex] (ap1) edge node[above, sloped] {$ s$} (a1); 
            \draw[-Latex] (a2) edge node[above, sloped] {$ r$} (a3);
            \draw[-Latex] (a2) edge node[above, sloped] {$ s$} (ap3);
            \draw[-Latex] (ap2) edge node[above, sloped] {$ s$} (a2);
            \draw[-Latex] (a3) edge node[above, sloped] {$ r$} (a4);
            \draw[-Latex] (a3) edge node[above, sloped] {$ s$} (ap4);
            \draw[-Latex] (ap3) edge node[above, sloped] {$ s$} (a3);
            \draw[-Latex] (ap4) edge node[above, sloped] {$ s$} (a4);

                    \node[label = ${ }$] (b) at (-2,-5.0) {$b$};    
                    \node[label = ${ }$] (b10) at (0,-4.0) {$b_{1,0}$};                    
                    \node[label = ${ }$] (b20) at (0,-6.0) {$b_{2,0}$};
  
                    \node[label = ${ }$] (bp21) at (0.75,-5.0) {$b_{2,1}'$};
                    \node[label = ${ }$] (b11) at (1.5,-4.0) {$b_{1,1}$};
                    \node[label = ${ }$] (b21) at (1.5,-6.0) {$b_{2,1}$};
                   
                    \node[label = ${ }$] (bp12) at (2.25,-3.0) {$b_{1,2}'$};
                    \node[label = ${ }$] (b12) at (3.0,-4.0) {$b_{1,2}$};
                    \node[label = ${ }$] (b22) at (3.0,-6.0) {$b_{2,2}$};
                    
                    \node[label = ${ }$] (bp23) at (3.75,-5.0) {$b_{2,3}'$};
                    \node[label = ${ }$] (b13) at (4.5,-4.0) {$b_{1,3}$};
                    \node[label = ${ }$] (b23) at (4.5,-6.0) {$b_{2,3}$};
                            
                    \node[label = ${ }$] (bp14) at (5.25,-3.0) {$b_{1,4}'$};
                    \node[label = ${A }$] (b14) at (6.0,-4.0) {$b_{1,4}$};
                    \node[label = ${A }$] (b24) at (6.0,-6.0) {$b_{2,4}$};
                                                
                    \draw[-Latex] (b) edge node[sloped, above left = 2pt] {$ r$} (b10); 
                    \draw[-Latex] (b) edge node[sloped, above left = 2pt] {$ r$} (b20); 
                    
                    \draw[-Latex] (b10) edge node[above] {$ r$} (b11); 
                  
                    \draw[-Latex] (bp21) edge node[above, sloped] {$ s$} (b21); 
                    \draw[-Latex] (b11) edge node[above, sloped] {$ r$} (b12); 
                    \draw[-Latex] (b11) edge node[above, sloped] {$ s$} (bp12); 
                    \draw[-Latex] (b20) edge node[above, sloped] {$ r$} (b21); 
                    \draw[-Latex] (b21) edge node[above, sloped] {$ r$} (b22); 
                                                        
                    \draw[-Latex] (bp12) edge node[above, sloped] {$ s$} (b12); 
                    \draw[-Latex] (b12) edge node[above, sloped] {$ r$} (b13); 
           
                    \draw[-Latex] (b22) edge node[above, sloped] {$ r$} (b23); 
                    \draw[-Latex] (b22) edge node[above, sloped] {$ s$} (bp23);              
                    \draw[-Latex] (bp23) edge node[above, sloped] {$ s$} (b23); 
                    \draw[-Latex] (b20) edge node[above, sloped] {$ s$} (bp21); 
                    \draw[-Latex] (b13) edge node[above, sloped] {$ r$} (b14); 
                    \draw[-Latex] (b13) edge node[above, sloped] {$ s$} (bp14); 
                    \draw[-Latex] (b23) edge node[above, sloped] {$ r$} (b24);                     
                   
                    \draw[-Latex] (bp14) edge node[above, sloped] {$ s$} (b14); 
                          
                    \draw[-Latex] (a) edge node[above, sloped] {$r$} (a0);
                    \draw[-Latex] (a) edge[bend right] node[above, sloped] {$r$} (b10);
                    \draw[-Latex] (a) edge[bend right] node[above, sloped] {$r$} (b20);

                    \node[right = of b14] (desc1) {$S_1=\{1,3\}$};
                    \node[right = of b24] (desc2) {$S_2=\{2,4\}$};
                    
                    \end{tikzpicture}
                    \end{center}
            \caption{Example of the interpretations $\Imc$ and $\Jmc$ used in the reduction.}
        \label{fig:np}
        \end{figure}

        We verify Equivalence~\eqref{eq:correctness} by showing that the following are equivalent: 

        \begin{enumerate}[label=(\roman*)]
        \item $S$ has a hitting set of size at most $k$; 
        \item there is an $\mathcal L(\{\exists\})$ concept fitting $P,N$ of size at most $k'$; 
        \item there is an $\ALC$ concept fitting $P,N$ of size at most $k'$.
        \end{enumerate}
        We start with (i)$\Rightarrow $(ii). 
        Let $H$ be a hitting set for $S=\{S_1,\dots,S_m\}$ with $|H| = k$. We inductively define concepts $C_i$, for $i=0,\ldots,n$, by setting $C_0=A$ and 
\begin{equation}\label{eq:paths}
       C_i = \begin{cases}
            \exists r.C_{i-1} & \text{if }n-i+1\notin H\\
            \exists s.\exists s.C_{i-1} & \text{otherwise},
        \end{cases}
\end{equation}
        for $1\leq i\le n$. Thus, $\lVert C_n \rVert = n + |H| +1$ and therefore $\lVert\exists r.C_n\rVert = n + k +2 = k'$. 
        We claim that $D=\exists r.C_n$ fits $P,N$. In the example in Figure~\ref{fig:np}, the concept constructed for the hitting set $H=\{1,2\}$ is $D=\exists r.\exists s.\exists s.\exists s.\exists s.\exists r.\exists r.A$, and it is easily verified that it fits. In general, any concept
        that can be obtained this way fits the positive example $(\Imc,a)$. Moreover, a concept $D$ constructed from a hitting set $H$ fits the negative example $(\Jmc,b)$ since the construction forces the path described by $D$ to ``leave'' the paths encoding the sets $S_j$ in $(\Jmc,b)$. This means that $D$ cannot be satisfied in $b$, since the sink $c$ does not satisfy $A$.

        The direction (ii)$\Rightarrow$(iii) is immediate. For
        (iii)$\Rightarrow$(i), we take advantage of the structure of
        $(\Imc,a)$ and $(\Jmc,b)$ to show that it is without loss of generality
        to assume that the smallest fitting is from $\Lmc(\{\exists\})$ and 
        has the shape from Equation~\eqref{eq:paths}. The main
        properties we exploit are that, except at the roots $a,b$, $r,s$ are interpreted as total functions (to replace $\forall r/s$ by $\exists r/s$) and that there are single positive and negative examples (to remove disjunction and conjunction, respectively).
        \qed
        
        \end{proof}

\section{Generalization}\label{sec:generalization}

In this section, we investigate the generalization ability of fitting algorithms with certain properties using the well-known PAC learning framework~\cite{DBLP:journals/cacm/Valiant84}, see also~\cite{anthony1997computational} for an introduction. Here, with a \emph{fitting algorithm for \Lmc}, we mean an algorithm $\mathcal A$ that receives as input $P,N$ and returns an $\Lmc$ concept $C$ that fits $P,N$ if such a concept exists. 
To define PAC learning, we need some  additional notation.

A \emph{probability distribution over examples} is a function $\mathbb P$ that maps examples $(\Imc,a)$ to a non-negative value $\mathbb P(\Imc,a)$ in a way such that 
$\sum_{(\Imc,a)}\mathbb P(\Imc,a)=1$. A \emph{signature} is a finite set $\Sigma$ of concept and role names. We say that $\mathbb P$ is \emph{over $\Sigma$} if in examples $(\Imc,a)$ 
with $\mathbb P(\Imc,a)>0$, the $\Imc$ contains only facts $A(a)$ and $r(a,b)$ with $A\in \Sigma$ and $r\in\Sigma$, respectively. A \emph{sample} of $\mathbb P$ is a set $\mathit{Ex}$ of examples drawn according to $\mathbb P$. \emph{Labeling $\mathit{Ex}$ according to a (target) concept $C_T$} means to split $\mathit{Ex}$ into $P,N$ by taking $P=\{(\Imc,a)\in Ex\mid a\in C_T^\Imc\}$ and $N=\mathit{Ex}\setminus P$. The \emph{error} of a (hypothesis) concept $C_H$ compared to $C_T$ is defined as the probability that $C_H$ and $C_T$ label a random example differently, that is: 
\[\mathrm{error}_\mathbb P(C_T,C_H)= \mathbb P(\{(\mathcal I,a)\mid a\in C_T^\Imc\mathop{\Delta} C_H^\Imc\})\]
where $\Delta$ denotes symmetric difference. We denote with $(0,1)$ the open interval between $0$ and $1$.

\begin{definition}
Let $O \subseteq \oall$.
    A fitting algorithm $\mathcal A$ for $\Lmc(O)$ is a PAC learning algorithm if there is a function $m \colon \mathbb{R}^2\times \mathbb N^3\rightarrow \mathbb N$, such that for all $\varepsilon,\delta\in(0,1)$, all signatures $\Sigma$, all $s,n\in\mathbb N$, all probability distributions $\mathbb P$ over examples $(\mathcal I,a)$ over $\Sigma$ with $\lVert(\mathcal I,a)\rVert\le s$, and all target concepts $C_T\in \Lmc(O)$ over $\Sigma$ with $\lVert C_T\rVert\le n$, algorithm $\mathcal A$ has the following property. If $\mathcal A$ receives $m(\frac 1\delta,\frac 1\varepsilon,|\Sigma|,s,n)$ examples $P,N$ sampled from $\mathbb P$ and labeled according to $C_T$, it returns with high probability at least $1-\delta$ a concept $C_H \in \Lmc(O)$ with small $\mathrm{error}_\mathbb P(C_T,C_H)\le\varepsilon$.

    A PAC learning algorithm is said to be \emph{sample-efficient} if $m$ is a polynomial, and $\emph{efficient}$ if it is sample-efficient and runs in polynomial time.
\end{definition}

Unfortunately, efficient PAC learnability is known to be elusive, and we confirm this to be the case here as well, under reasonable assumptions.
\begin{restatable}{theorem}{thmefficient}\label{thm:efficient} Let $O\subseteq \oall$. If there is an efficient PAC learning algorithm for $\Lmc(O)$, then:
\begin{enumerate}

    \item $\NPclass=\RPclass$, if $O$ contains at least one of $\exists/\forall$ and $\{\sqcap,\sqcup\}\not\subseteq O$;

    \item RSA encryption is polynomial time invertible, if $\{\sqcap,\sqcup\}\subseteq O$.
\end{enumerate}
\end{restatable}
Recall that $\RPclass$ is the class of problems for which probabilistic polynomial-time Turing machines exist that reject no-instances, and accept yes instances with probability $\frac{1}{2}$.
Both $\NPclass=\RPclass$ and polynomial time RSA decryption (without private key) are unlikely.  Point~1 was known for $\EL=\Lmc(\{\exists,\sqcap\})$~\cite{DBLP:journals/ipl/CateFJL24, DBLP:conf/ecml/Kietz93} and we show it to be the case for all other mentioned fragments. The proof relies on a well-known result due to Pitt and Valiant which roughly states that efficient PAC learnability implies (under suitable conditions) that the corresponding (not size restricted) fitting problem is in \RPclass~\cite{PittValiant88}.
To show Theorem~\ref{thm:efficient}, we thus consider the relevant fitting problems.
In the appendix, we show that the fitting problem for $\Lmc(\{\forall\})$, $\Lmc(\{\forall,\sqcap\})$, and $\Lmc(\{\forall,\exists,\sqcap\})$ is $\NPclass$-hard via reduction from SAT. Applying Lemma~\ref{lem:fitting_dual} then yields the same lower bounds for all other fragments from Point~1. The same strategy fails for $O$ as in Point~2 as the fitting problem for such $O$ is solvable in polynomial time. Instead, we observe that $\Lmc(O)$ contains all monotone Boolean formulas and exploit that these are not efficiently PAC learnable under cryptographic assumptions~\cite{DBLP:journals/jacm/KearnsV94}.

Given these negative results on \emph{efficient} PAC learning, we will analyze concrete fitting algorithms regarding their sample-efficiency. More precisely, we study algorithms that behave well from a logical perspective in that they return fitting concepts that are  most specific, most general, or of minimal quantifier depth among all fitting concepts. We denote with $C\sqsubseteq D$ the fact that $C^\Imc\subseteq D^\Imc$, for all interpretations~\Imc, and let $C\equiv D$ abbreviate $C\sqsubseteq D$ and $D\sqsubseteq C$. Recall that a concept $C$ is a \emph{most specific fitting} for $P,N$ if it fits $P,N$ and there is no concept $D$ with $D\sqsubseteq C$ and $D\not\equiv C$ that fits. Dually, a concept $C$ is a \emph{most general fitting} concept for $P,N$ if it fits $P,N$ and there exists no $D$ with $C\sqsubseteq D$ and $C\not\equiv D$ that fits as well. Finally, the \emph{quantifier depth} of a concept is the nesting depth of $\exists/\forall$. Unfortunately, in most cases, logically well-behaved algorithms are not sample-efficient. The following theorem summarizes the situation.
\begin{theorem}\label{thm:generalization}
Let $O\subseteq\oall$ be any set containing at least one of $\exists/\forall$ and at least one of $\sqcap/\sqcup$, and let $\Amc$ be a fitting algorithm for $\Lmc(O)$. Then \Amc is not a sample-efficient PAC learning algorithm, if:
\begin{enumerate}
    \item $O\neq\{\exists,\sqcup\}$ and \Amc always returns a most specific fitting if one exists; 
    
    \item $O\neq\{\forall,\sqcap\}$ and \Amc always returns a most general fitting if one exists;
    
    \item \Amc always returns a fitting of minimal quantifier depth if some fitting exists.
\end{enumerate}
\end{theorem}

It is interesting to note that Lemma~\ref{lem:fitting_dual} plays a central role here as well. 
Indeed, the dualization operation preserves quantifier depth and turns most specific fittings into most general ones, and vice versa. This means that it suffices to show only one of Point~1 or Point~2; the other one follows by dualization.
Theorem~\ref{thm:generalization} was known for $\Lmc(\{\sqcap,\exists\})=\EL$~\cite{DBLP:conf/ijcai/CateFJL23}. We begin with Point~3 and note that it suffices to show Lemma~\ref{lem:not_sample_efficient_depth} below and then apply Lemma~\ref{lem:fitting_dual}.
\begin{restatable}{lemma}
{lemnotsampleefficientdepth}\label{lem:not_sample_efficient_depth} For any $O$ with $\{\exists,\sqcup\} \subseteq O\subseteq \oall$, any fitting algorithm for $\mathcal L(O)$ that always returns a fitting of minimal quantifier depth if some fitting exists is not a sample-efficient PAC learning algorithm.
\end{restatable} 

\begin{proof}
    Assume there is a sample-efficient PAC learning algorithm $\mathcal A$ with associated polynomial $m$ that always returns a fitting $\mathcal L(O)$ concept of minimal quantifier depth, if some fitting exists. We choose $\varepsilon=0.4$, $\delta=0.5$, $\Sigma=\{r,s,t\}$, and a target concept $C_T=\exists t^{n+1}.\top$, for some $n$ to be determined later. We consider the following interpretations $\Imc_w,\Jmc_w$, for $w\in \{r,s\}^n$: 
    \begin{itemize}
        \item for each $w\in \{r,s\}^n$, $\Imc_w$ is a path of length $n$ whose edges are labeled according to $w$ and whose start element is called $a_w$;

        \item for each $w\in \{r,s\}^n$, $\Jmc_w$ is obtained from $\Imc_w$ by adding a $t$-path of length $n+1$ starting from $a_w$.
    \end{itemize}
    Clearly, each $(\Jmc_w,a_w)$ is a positive example, while each $(\Imc_w,a_w)$ is a negative example. 

   Let $\mathbb P$ be the probability distribution that assigns probability $\frac 1{2^{n+1}}$ to each of the $2^{n+1}$ examples $(\Imc_w,a_w),(\Jmc_w, a_w)$, and $0$ to all others. We choose $n$ large enough so that the probability that a sample of size $m(\frac 1\delta,\frac 1\varepsilon,3,2n+1,n+2)$ does not contain both $(\Imc_w,a_w)$ and $(\Jmc_w,a_w)$ for some $w\in\{r,s\}^n$ is at least $1-\delta$.
    
    Let $P,N$ be a sample of size $m(\frac 1\delta,\frac 1\varepsilon,3,2n+1,n+2)$ and suppose that it satisfies the mentioned property: for each $w\in\{r,s\}^n$, not both $(\Imc_w,a_w)\in N$ and $(\Jmc_w,a_w)\in P$. Consider the $\mathcal L(\{\exists,\sqcup\})$ concept \[C_0=\bigsqcup\limits_{(\mathcal I_w,a_w)\in P} C_{w}\] where $C_{w}$ is the concept $\exists t_1\ldots\exists t_n.\top$ if $w=t_1\dots t_n$. Clearly, $C_0$ has quantifier depth $n$ and fits $P,N$. Hence, if $\mathcal A$ runs on $P,N$, it returns a fitting $\mathcal L(O)$ concept $C$ quantifier depth at most $n$. However, such $C$ cannot distinguish $(\Imc_w,a_w)$ and $(\Jmc_w,a_w)$, for any $w\in\{r,s\}^n$, hence $\text{error}_{\mathbb P}(C_T,C)=0.5>\varepsilon$.\qed
\end{proof}

We now turn our attention to Points~1 and~2 of Theorem~\ref{thm:generalization}. Inspecting the proof of Point~1 for \EL from~\cite[Theorem~6]{IJCAI23arxiv} reveals that it actually implies Point~1 for each fragment of \ALC that contains~\EL. Hence, it suffices to show Point~2 for $\{\exists,\sqcup\}$, $\{\exists,\sqcup,\sqcap\}$, and $\{\exists,\forall,\sqcap\}$ (note, that adding negation to any of these operator sets leads to full \ALC). We show the first here and the latter two in the appendix. Both are by non-trivial extensions of the proof for the same fact for $\EL$.
\begin{restatable}{lemma}{lemmostgeneralexistscup}
\label{lem:mostgeneral-exists-cup}
    Any fitting algorithm for $\mathcal L(\{\exists, \sqcup\})$ that always returns a most general fitting concept if one exists is not 
    sample-efficient.   
    \end{restatable}
\begin{proof}
Assume towards a contradiction that there is a fitting algorithm $\mathcal A$ for $\Lmc(\{\exists,\sqcup\})$ that always returns a most general fitting $\mathcal L(\{\exists, \sqcup\})$ concept and which moreover
is sample-efficient with associated polynomial $m$. We choose $\varepsilon =\frac 17, \delta= \frac 89$, $\Sigma=\{r,s,A\}$, $C_T=A$, and $n$ large enough such that $2^{n-1}>m(\frac 1\varepsilon,\frac 1\delta,3,2n,1)$. We use the following interpretations $\Imc,\Jmc$ and $\Imc_w$ for $w\in \{r,s\}^n$:
\begin{itemize}

    \item \Imc consists of elements $a_1,\ldots,a_n$ satisfying $a_1\in A^\Imc$, $a_n\in A^\Imc$, and $(a_i,a_{i+1})\in r^\Imc$ and $(a_i,a_{i+1})\in s^\Imc$, for all $i<n$.
    
    \item $\mathcal J$ consists of elements $b_1,\dots,b_{n},b_{n+1}$ satisfying $b_i\in A^\Jmc$ for all $i\in\{2,\ldots,n-1,n+1\}$ and 
    $(b_i,b_{i+1})\in r^\Imc$, $(b_i,b_{i+1})\in s^\Imc$, for all $i\leq n$, and $(b_{n+1},b_{n+1})\in r^\Jmc$ and $(b_{n+1},b_{n+1})\in s^\Jmc$. 
    
    \item For $w\in\{r,s\}^n$, $\Imc_w$ is a path of length $n$ whose edges are labeled according to $w$, whose end satisfies $A$, and whose start element is denoted $a_w$.
\end{itemize}
Let $\mathbb P$ be the  probability distribution that assigns 
$\mathbb P(\mathcal I,a_1)=\mathbb P(\mathcal J,b_1)=\frac 13$, $\mathbb P(\mathcal I_w,a_w)=\frac{1}{3\cdot2^n}$ for $w\in\{r,s\}^n$ and probability $0$ to all other examples. Note that $(\Imc,a_1)$ is a positive example, and $(\Jmc,b_1)$ and all $(\Imc_w,a_w)$ are negative examples for $C_T$. 

Let $P,N$ be a sample of size $m(\frac 1\varepsilon, \frac 1\delta, 3,2n,1)$ drawn according to $\mathbb P$. Then with probability at least $1-\delta$, both $(\mathcal I,a_1)\in P$ and $(\mathcal J,b_1)\in N$.
We argue in the appendix that $C_0=A\sqcup \bigsqcup_{(\mathcal I_w,a_w)\notin N} C_{w}$ is the unique most general $\Lmc(\{\exists,\sqcup\})$ concept fitting $P,N$, where $C_w$ is $\exists t_1.\ldots.\exists t_n.A$ in case $w=t_1\cdots t_n$.
Thus, $\mathcal A$ returns (a concept equivalent to) $C_0$ when run on $P,N$. However, all negative examples $(\Imc_w,a_w)\notin N$ are labeled incorrectly by $C_0$. The choice of $n$ implies that there are at least $2^{n-1}$ such examples; thus the error is at least $\frac 16>\varepsilon$. \qed
\end{proof}

It remains to discuss the exceptions mentioned in Points~1 and~2 of Theorem~\ref{thm:generalization}. Interestingly, bounded fitting is a counter-example for these cases: 
\begin{restatable}{lemma}{lemboundedfittingexistsor}\label{lem:boundedfittingexistsor}
Bounded fitting for $\mathcal L(\{\exists, \sqcup\})$ (resp., $\Lmc(\{\forall,\sqcap\})$ returns a most specific $\mathcal L(\{\exists,\sqcup\})$ (resp. most general $\Lmc(\{\forall,\sqcap\})$) concept, if any fitting exists.
\end{restatable}

\section{SAT-Based Bounded Fitting}\label{sec:implementation}
\newcommand{\sigr}{\ensuremath{\Sigma_{\textsf{R}}}\xspace}
\newcommand{\sigc}{\ensuremath{\Sigma_{\textsf{C}}}\xspace}

Motivated by the \NPclass-completeness of size-restricted fitting established in
Theorem~\ref{thm:main1}, we implemented bounded fitting for any logic $\mathcal
L(O)$ for $O\subseteq\oall$ using a SAT solver to efficiently search for fitting
$\Lmc(O)$ concepts of a given size. In this section, we first give an overview
of the reduction to SAT and describe two optimizations that improve the
performance of the basic reduction. Then, we propose a variant of bounded
fitting tailored towards approximate fitting. Finally, we compare our
implementation with other concept learning implementations.

We describe the reduction for \ALC. Let $P, N$ be sets of examples and $k \geq 1$. We aim to produce a propositional formula $\varphi$ that is satisfiable if and only if there is an $\ALC$ concept $C$ that fits $P$, $N$ with $\lVert C \rVert \leq k$.
We assume w.l.o.g.\ that all examples in $P$ and $N$ share a single (finite) interpretation $\Imc$. 
Let $\sigc$ and $\sigr$ be the set of concept and role names, respectively, that occur in $\Imc$. 

The formula $\varphi$ is the conjunction of two formulas $\varphi_1\wedge\varphi_2$ whose purpose is to encode syntax trees of $\ALC$ concepts and to evaluate the encoded concepts over $\Imc$, respectively.
For this task, $\varphi_1$ uses 

variables $x_{i,v}$, $y_{1,i,j}$, and $y_{2, i, j}$ for $i, j \in \{1, \ldots, k\}$ and $v$ a node label from $\Vmc = \{\top,\bot,\neg,\sqcap,\sqcup\} \cup \sigc \cup \{\exists r, \forall r\mid r\in \sigr\}$
with the following intended meaning:
\begin{itemize}

\item $x_{i, v}$ is true if and only if node $i$ of the syntax tree is labeled with $v$, 

\item $y_{1,i,j}$ is true if and only if node $i$ has the single successor $j$, and

\item $y_{2, i, j}$ is true if and only if node $i$ has the two successors $j$ and $j + 1$.
\end{itemize}
We treat node $1$ as the root of the syntax tree.
The clauses in $\varphi_1$ enforce that a satisfying assignment to these variables represents a well-formed syntax tree of an $\ALC$ concept. For example, to ensure that nodes that are labeled by a concept name do not have successors and that nodes labeled with $\sqcap$ have two successors, respectively, the clauses 
\[
x_{i, A} \to (\neg y_{1, i, j} \land \neg y_{2, i, j}) \quad\quad\text{and}\quad\quad
x_{i, \sqcap} \to\textstyle \bigvee_{i < \ell < k} v_{2, i, \ell}
\]
are included in $\varphi_1$ for each $i, j \in \{1, \ldots, k\}$ with $i < j$ and $A \in \sigc$. Similar clauses encode the rules for nodes labeled with other labels from~\Vmc. In case we are interested in a fragment $\Lmc(O)$ of \ALC, we restrict $\mathcal V$ accordingly.

To evaluate the encoded \ALC concept, $\varphi_2$ uses variables $z_{i, a}$ for $i \in \{1, \ldots k\}$, $a \in \Delta^{\Imc}$ with the intended meaning that $z_{i, a}$ is true if and only if 
$a \in C^\Imc$ where $C$ is the concept represented by node $i$ in the syntax tree.
For example, to encode the semantics of concept names and of conjunction~$\sqcap$, $\varphi_2$ contains the clauses
\begin{align} 
   x_{i,A} \to z_{i,b},\quad  x_{i,A} \to \neg z_{i,c}, \quad
  x_{i, \sqcap} \land y_{2, i, j} \to \left(z_{i, a} \leftrightarrow (z_{j, a} \land z_{j + 1, a})\right) \label{eq:cn-clauses}
\end{align}
for all $i,j \in \{1, \ldots, k\}$, $A \in \sigc$, $b\in A^\Imc$, $c\in \Delta^\Imc\setminus A^\Imc$, and $a\in \Delta^\Imc$. 

Finally, the fitting condition is enforced by the following formula:
\begin{equation} \label{eq:fitting-clauses}
 \textstyle\bigwedge_{(\Imc,a) \in P}z_{1,a} \land \bigwedge_{(\Imc,b)\in N}\neg z_{1,b}
\end{equation}
Given the provided intuitions, it should be clear that we can read off a witness concept from a satisfying assignment of $\varphi$.
Overall, the constructed formula $\varphi$ uses $O(k^2 + k \lvert \Sigma \rvert + k n)$ variables, $O(nk^3\lvert \Sigma \rvert )$ clauses, and is of total size $O(nk^3 \lvert \Sigma\rvert )$ where $\Sigma = \sigc\cup\sigr$ and $n=\lVert \Imc\rVert$.

\subsection{Optimizations}

We discuss two optimizations that we implemented on top of the basic reduction described above. Similar optimizations have been implemented in~\cite{DBLP:conf/ijcai/CateFJL23, DBLP:conf/fdl/Riener19}.

\paragraph{Taking Advantage of Types} The encoding of the semantics of concept names given in \eqref{eq:cn-clauses} produces $k \times |\Delta^{\Imc}| \times |\sigc|$ clauses.
In realistic data sets, both $|\Delta^{\Imc}|$ and $|\sigc|$ tend to be large. For example, in the Carcinogenesis data set of the SML-Benchmark~\cite{DBLP:journals/semweb/WestphalBBJL19}, there are 19\,221 individuals and 133 concept names, resulting in 
more than 10 million clauses for $k = 4$, which is the majority of clauses in the encoding. To reduce the number of clauses, one can exploit that not all combinations of concept names actually occur in realistic data sets. In \emph{carcinogenesis}, only 105 different combinations occur. This allows us to encode the same semantics with fewer clauses, at the cost of introducing additional variables. 

The \emph{type} $\mathsf{type}(a)$ of an individual $a$ in an interpretation $\Imc$ is the set $\{ A \in \sigc \mid a\in A^\Imc \}$.
Let $T$ be the set of all types that occur in $\Imc$, that is, $T = \{ \mathsf{type}(a) \mid a \in \Delta^{\Imc}\}$.
For each $i \in \{1, \ldots, k\}$ and $t \in T$, we introduce an additional variable $x_{i, t}$
with the intended meaning that $x_{i,t}$ is true if and only if 
node $i$ is labeled with a concept name that is contained in $t$. We ensure this with the clauses
\[
x_{i, A} \to x_{i, t} \quad \text{and} \quad x_{i, B} \to \neg x_{i, t} 
\]
for each $t \in T$, $i \in \{1, \ldots, k\}$, $A \in t$ and $B \in \sigc \setminus t$.
We then replace the clauses that enforce the semantics of concept names with
\[
 x_{i, \mathsf{type}(a)} \to z_{i, a} \quad \text{and} \quad \neg x_{i, \mathsf{type}(a)} \to \neg z_{i, a} \lor \neg \ell_i
\]
for all $i \in \{1, \ldots, k\}$ and $a \in \Delta^{\Imc}$, 
where $\ell_i$ are additional variables that are true if and only if node $i$ is labeled with a concept name. In the carcinogenesis dataset, the modified encoding produces only $4 \cdot 105 \cdot 133 + 4 \cdot 19\,221 \cdot 2 = 209\,628$ clauses. 

\paragraph{Syntax Tree Templates} A major problem is that different models of $\varphi_1$ can represent (syntax trees of) equivalent concepts. For example, already the arguably simple concept $\exists r.A \sqcap \forall s.B$ is encoded in four different models. This phenomenon is well-known also from other domains and is called \emph{symmetry} in the search space~\cite{DBLP:series/faia/Sakallah21}. The problem it brings is that the SAT solver often does significant redundant work by trying symmetric solutions, which slows down the solving process, especially for unsatisfiable formulas.

A solution is to \emph{break} symmetries by introducing additional clauses to the encoding that restrict how the nodes of a syntax tree may be ordered. Similar symmetries have been observed in other SAT-based implementations of bounded fitting~\cite{DBLP:conf/ijcai/CateFJL23,DBLP:conf/ijcar/PommelletSS24}.

By the \emph{topology} of the syntax tree of an $\ALC$ concept, we mean the binary tree that results from removing all node labels. It is clear that isomorphic topologies can represent the same set of concepts (up to commutativity of $\sqcap/\sqcup$).
We thus extend $\varphi$ by clauses that enforce, for each topology isomorphism equivalence class, a canonical ordering of the nodes in the syntax tree.
This is achieved by enumerating all binary trees with $k$ nodes, and through a long disjunction, demanding that the encoded syntax tree must have one of these topologies. A similar approach has been taken in~\cite{DBLP:conf/fdl/Riener19}, but instead of including the large disjunction in the encoding, there is one call of the SAT solver per topology. 

Note that the number of binary trees with $k$ nodes grows exponentially with $k$. To prevent a negative impact on runtime,
we limit this process to topologies with at most $10$ nodes, after which the topologies are only considered as \emph{prefixes} of the topology of the encoded syntax tree. The used threshold of 10 was determined empirically; higher values resulted in a slowdown on the benchmarks.

It is worth noting that these clauses do not break all symmetries, especially
many related to commutativity remain. Essentially, there will always be a
tradeoff between the number of clauses needed to break symmetries and the gain
in avoiding symmetric solutions. In the extreme case, breaking \emph{all}
symmetries will require many clauses and thus likely not result in performance
improvements. In our experience, adding some extra clauses avoiding certain syntactical patterns proved beneficial, e.g., $\top$ or $\bot$ directly below $\sqcap$ or $\sqcup$, $\neg$ directly below $\neg$, $\exists$, $\forall$,
and two $\sqcap$-successors below a $\sqcap$-node, and the same for $\sqcup$.

\subsection{Approximation Scheme}\label{sec:approx}

In practice, there may not always be a concept that perfectly fits the provided examples. In such cases, it is more realistic to view the fitting problem as an optimization problem. We address this by changing the basic bounded fitting algorithm from Algorithm~\ref{alg:boundedfitting} into the approximation scheme described in Algorithm~\ref{alg:approximation}. There, an extra variable $m$ stores the best coverage of examples achieved so far. If an exact fitting concept is found (Lines~4/5), the algorithm returns it. Otherwise, $m$ is increased (Line~7). If for the current size-bound $k$, $m$ cannot be improved (test in Line~3), the algorithm moves on to the next $k$. Note that this approximation scheme is an \emph{anytime algorithm} in the sense that, if it is interrupted at some point, it may return the best solution found up to this point.

This change in the algorithm is reflected in the interface to the SAT solver as follows. 
We replace clauses~\eqref{eq:fitting-clauses}, that demand exact fitting, with \emph{cardinality constraints}, that demand that at least $m$ of the literals in~\eqref{eq:fitting-clauses} be true. 
We exploit the ability to solve \emph{incrementally} provided by some SAT solvers, since the increase of $m$ can be realized by adding clauses. However, whenever $k$ increases, we have to produce a new propositional formula as before.

\begin{algorithm}[t]
\caption{Approximation version of Bounded Fitting.}\label{alg:approximation} \KwIn{Positive examples $P$, negative examples $N$}
   $m \gets 1$\;
   \For {$k:=1,2,\ldots$}{
        \While{there is $\varphi\in \Lmc$ of size $k$ that fits $m$ examples from $P,N$}{
        \If(\hfill // all examples covered?) {$m=|P|+|N|$}{
         \Return $\varphi$
        } \Else(\hfill // try to cover more examples) { increase $m$\;}
        }
   }
\end{algorithm}

\subsection{Evaluation}

Our implementation uses the PySat library~\cite{imms-sat18,itk-sat24} to decide satisfiability of $\varphi$ with the Glucose SAT solver and, if satisfiable, read off a fitting concept from a model of $\varphi$.
In the following, we refer with ALC-SAT to the base implementation and with ALC-SAT$^+$ to the one with
the optimized encoding.

We compared our implementation to other concept learning tools CELOE~\cite{DBLP:conf/www/BuhmannLWB18}, SParCEL~\cite{DBLP:journals/jmlr/TranDGM17}, and EvoLearner~\cite{DBLP:conf/www/HeindorfBDWGDN22}.\footnote{We use the version of CELOE from DL-Learner 1.5.0 and the version of EvoLearner from ontolearn 0.8.1. The exact learning benchmarks were run on a MacBook Pro with M3 Pro chip and 36GB RAM. The generalization benchmarks were run on a MacBook Pro with M1 Pro Chip and 32 GB
RAM.} While these tools also support extensions of $\ALC$ such as number restrictions or data properties, we restricted all tools to $\ALC$ for our evaluation. We did not consider fragments in these first experiments since not all tools support them.
We also did not compare our implementation against concept learning tools such as NCES2~\cite{DBLP:conf/pkdd/KouagouHDN23} and DRILL~\cite{DBLP:conf/ijcai/DemirN23}, as these require a pretraining step for each data set, which makes direct performance comparison difficult. 
We evaluate performance in two setups: in the first, we test the ability to find exact fitting concepts on a new benchmark where fitting concepts are guaranteed to exist. In the second, we evaluate the approximation scheme from Section~\ref{sec:approx} 
on the standard SML benchmark~\cite{DBLP:journals/semweb/WestphalBBJL19}. 

For evaluating exact fitting we extracted fragments of the YAGO 4.5 knowledge base~\cite{DBLP:conf/sigir/SuchanekABCPS24} by taking all facts talking about a given signature. The \emph{family} fragment is based on signature $\{\textit{spouse, children, gender, Male, Female}\}$, and the \emph{language} fragment uses the role name $\textit{hasLanguage}$ and several concept names that represent languages. 
We handcrafted 70 \ALC concepts of varying complexity and used them to obtain positive and negative examples by querying the extracted fragments. The task was then to reverse-engineer the \ALC concept from (subsets of) the examples.

For the family fragment we used concepts with high quantifier depth, for example 
$\exists c.(\exists c.\exists g.F\sqcap\exists c.\exists g.M)$ asking for all individuals having a child that has both a daughter (gender female) and a son (male); we use also deeper concepts, referring to grand-children and so on. 

For the language fragment we use shallow concepts that contain Boolean combinations of concept names under existential or universal restrictions, for example, $\exists l.(F\sqcup I\sqcup R)\sqcap\forall l.\neg G$.

\begin{figure}[t]
\centering
\begin{subfigure}[t]{0.5\textwidth}
\centering
\scriptsize
    \begin{tikzpicture}
        \begin{axis}[
                xmin = 0,
                xmax = 5,
                ymin = 0,
                ymax = 5,
                width=5cm, height=5cm,
                grid = major,
                grid style={dashed, gray!30},
                ylabel=time ALC-SAT (seconds),
                xlabel=time ALC-SAT$^+$ (seconds),
                legend style={at={(0.1,-0.1)}, anchor=north}
             ]                     
            \addplot[only marks, mark =x] table [col sep=comma, y = t_alcsat, x=t_alcsat+] {data/family_alcsat_alcsat+_time_language.csv};
            \addplot [red, smooth, domain = 0:100] {x};
        \end{axis}
    \end{tikzpicture}
\end{subfigure}~%
\begin{subfigure}[t]{0.5\textwidth}
\centering
\scriptsize
    \begin{tikzpicture}
        \begin{axis}[
                xmin = 0,
                xmax = 240,
                ymin = 0,
                ymax = 50,
                width=5cm, height=5cm,
                grid = major,
                grid style={dashed, gray!30},
                ylabel=\# exact fittings found,
                xlabel=time (seconds),
                legend style={nodes={scale=0.6, transform shape},at={(1,0.2)}, anchor=east}
             ]                     
            \addplot[mark =x, draw = red] table [col sep=comma, y = n_sparcel, x=x] {data/data_graph.csv};
            \addplot[mark =x, draw = blue] table [col sep=comma, y = n_alcsat, x=x] {data/data_graph.csv};               
            \addplot[mark =x, draw = yellow] table [col sep=comma, y = n_celoe, x=x] {data/data_graph.csv}; 
            \addplot[mark =x, draw = green] table [col sep=comma, y = n_evo, x=x] {data/data_graph.csv};              
            \addlegendentry{SpaRCEL}
            \addlegendentry{ALC-SAT$^+$}
            \addlegendentry{CELOE}
            \addlegendentry{EvoLearner}
        \end{axis}
    \end{tikzpicture}    
\end{subfigure}
\caption{Exact learning runtime}
\label{fig:eval_exact_runtime}
\end{figure}
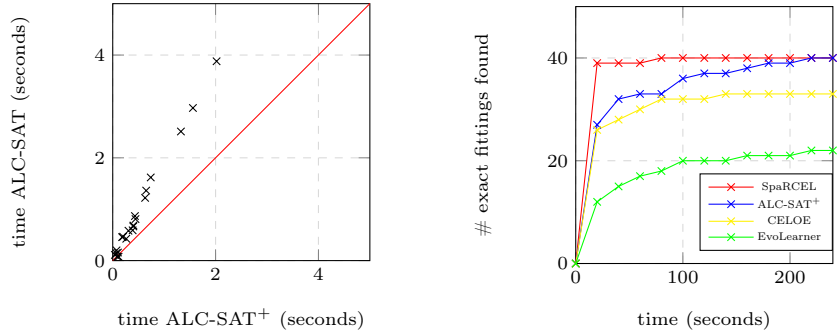

We highlight the impact of our optimizations on the encoding in Figure~\ref{fig:eval_exact_runtime} (left). The data points stem from experiments on the 25 benchmarks over the language fragment and compare the time ALC-SAT and ALC-SAT$^+$ needed until an exact fitting was found. 
The data shows that we find exact fitting concepts about twice as fast with optimizations enabled compared. We conjecture that this largely explained by the type encoding which is beneficial in the presence of many concept names.

We also compared the time of ALC-SAT$^+$ to find a fitting concept to that of EvoLearner, CELOE and SParCEL over the 45 benchmarks generated from the family fragment. The graph 
in Figure~\ref{fig:eval_exact_runtime} (right) displays the number of exact fittings found by each tool within a given runtime. It can be seen that ALC-SAT$^+$ finds exact fittings faster than both CELOE or EvoLearner. However, SParCEL achieves better results for smaller values for maximum runtime. Both SParCEL and ALC-SAT$^+$ find the same number of exact fittings within the timeframe displayed, however, for $5$ benchmarks SParCEL was unable to find an exact fitting concept whereas ALC-SAT$^+$ was able to find exact fittings for all benchmarks, at the cost of increased execution time. 

We evaluated approximate fitting of ALC-SAT$^+$ and the other tools by measuring accuracy and length of the returned concept on the SML-Benchmark using $10$-fold validation. Each tool was allowed to run for $5$ minutes. The results are given in Table~\ref{tab:sml_results} where for each tool and benchmark the first row shows accuracy while the second row shows concept length, both with standard deviation. The accuracy achieved by ALC-SAT$^+$ is comparable to that of EvoLearner, except for Mammographic. The values for accuracy of CELOE and SParCEL, in general, are lower. For SParCEL, a possible explanation is that the length of concepts is much larger than that of concepts returned by EvoLearner, CELOE and ALC-SAT$^+$. Small concepts are known to generalize well whereas large concepts are prone to overfitting. In comparison to SParCEL, EvoLearner, CELOE, and ALC-SAT$^+$ are able to find small fitting concepts. While this is `by design' for ALC-SAT$^+$, we conjecture that the heuristics in EvoLearner and CELOE prefer shorter concepts as well. Note, that when searching for approximate fitting concepts, ALC-SAT$^+$ returns the smallest concept with the highest accuracy found within the time limit and may thus report larger concepts than EvoLearner or CELOE. 

\begin{table}[t]
\caption{Generalization results on SML-Benchmarks}
\label{tab:sml_results}
\scriptsize
\begin{center}
\begin{tabular}{lcccccc}
\toprule
& \tiny Carcinogenesis & \tiny Hepatitis & \tiny Lymphography & \tiny Mammographic & \tiny Mutagenesis & \tiny Nctrer \\\midrule
\multirow{2}{*}{EvoLearner} &   $0.53\pm 0.18$  &   $0.58\pm 0.01$  &   $0.81\pm 0.12$   &   $0.46\pm 0.00$  &   $0.78\pm 0.18$  & $0.6\pm 0.04$  \\
& $7.2\pm 3.52$ &  $3.2,\pm 1.03$ & $19.6\pm 6.88$ & $1.7\pm 0.48$ & $3.5\pm 1.35$ & $3.3\pm 0.95$\\\midrule
\multirow{2}{*}{CELOE}&$0.54\pm 0.01$ &$0.41\pm 0.01$  & $0.82\pm 0.11$& $0.46\pm 0.0$& $0.56\pm 0.2$ & $0.6\pm 0.04$ \\
& $3.8\pm 0.42$ & $4.6\pm 1.26$&$10.8\pm 0.42$ &$1.7\pm 2.21$ & $6.8\pm 0.63$ &$3.9\pm 0.32$ \\\midrule
\multirow{2}{*}{SParCEL}& $0.54\pm 0.1$& n/a &$0.74\pm 0.13$ &$0.55\pm 0.02$ &$0.73\pm 0.23$ &$0.46\pm 0.08$ \\  
& $860.1\pm 66.54$ & n/a & $164.3\pm 49.48$ &$178.2\pm 21.6$ & $85.1\pm 9.35$ & $65.4\pm 18.84$\\\midrule
\multirow{2}{*}{ALC-SAT$^+$} & $0.55\pm 0.2$ & $0.58\pm 0.01$& $0.8\pm 0.09$& $0.77\pm 0.05$& $0.81\pm 0.26$ & $0.63\pm 0.07$\\
& $6.4\pm 0.7$ &$4\pm 1.7$ & $9.9\pm 0.32$ & $11.3\pm 2.06$  & $9.4\pm 0.84$ & $11.7\pm 0.95$\\
\bottomrule
\end{tabular}
\end{center}
\end{table}

\section{Conclusion and Future Work}\label{sec:conclusion}

We investigated the bounded fitting paradigm for learning class expressions formulated in the description logic \ALC and its fragments. We analyzed theoretical properties, and we
believe that many of our results can be extended to more expressive description logics. 
Additionally, we provided a preliminary implementation that uses the power of a SAT solver to find good fittings, and compared it with other concept learning tools with encouraging results. 
From a theoretical perspective, it would be interesting to study bounded fitting in an open world setting where an ontology provides additional background knowledge.

From a practical perspective, we see three relevant avenues for future work. First, OWL 2 DL supports several practically relevant features that are not currently supported by our implementation, such as inverse roles, number restrictions, and data properties. Hence, one natural next step is to
extend our implementation to be able to learn (ideally) $\SROIQ$ concepts. Second, we would like to understand the syntactic shape of fittings that occur in realistic knowledge bases. While currently the shape of the sought concepts is determined by the allowed logical constructors, it would also be possible to further restrict it by bounding the number of, say, disjunctions. 
Finally, we would like to improve the speed of our implementation either by providing better CNF encodings or by \emph{parallelization} relying on syntax tree templates similar to~\cite{DBLP:conf/fdl/Riener19}.

\paragraph*{Supplemental Material Statement:} 
The source code of our implementation ALC-SAT$^+$ and instructions on how to reproduce the benchmarks are available at \url{https://github.com/SAT-based-Concept-Learning/ALCSAT}.

\bibliographystyle{splncs04}
\bibliography{main}

\newpage
\appendix

\section{Proofs for Section~\ref{sec:prelims}}
\begin{lemma}\label{lem_dual}
For all signatures $\Sigma$, $\ALC$ concepts $C$ with signature contained in $\Sigma$, interpretations $\Imc$, and $a\in \Delta^\Imc$: 
    \[a\in C^{\mathcal I}\Leftrightarrow a\notin \overline{C}^{\overline\Imc_\Sigma}.\]
\end{lemma}

\begin{proof}
  We show the lemma by structural induction.

\textsl{Base Cases.} Let $\mathcal I$ be an interpretation and $a\in\Delta^{\mathcal I}$. Concepts in the base case have the form $\top,\bot$ or $A$ for a concept name $A\in\Sigma\cap\NC$. In the case of $\bot$ or $\top$, $a\in\bot^{\mathcal I}\Leftrightarrow a\notin\top^{\overline\Imc_\Sigma}$ holds because of $\bot^{\mathcal J} = \emptyset$ and $\top^{\mathcal J} = \Delta^{\mathcal J}$ for any interpretation $\mathcal J$. In the case of a concept name $A$, $a\in A^{\mathcal I}\Leftrightarrow a\notin A^{\overline\Imc_\Sigma}$ follows from the definition of $\overline{\,\cdot\,}$. 

\textsl{Induction Step.} We distinguish cases: 

    \begin{enumerate}
    \item $C= \neg D$:
        \begin{eqnarray*}
            &&a\in (\neg D)^{\mathcal I}\\
            &\Leftrightarrow&a\notin D^{\mathcal I}\\
            &\stackrel{\text{(I.H)}}{\Leftrightarrow}&a\in \overline{D}^{\overline\Imc_\Sigma}\\
            &\Leftrightarrow& a\notin (\neg \overline D)^{\overline \Imc_\Sigma} =  (\overline{\neg D})^{\overline\Imc_\Sigma}
            \end{eqnarray*} 
        \item $C = (C_1\sqcap C_2)$:
        \begin{eqnarray*}
            &&a\in (C_1\sqcap C_2)^{\mathcal I}\\
            &\Leftrightarrow&a\in C_1^{\mathcal I}\cap C_2^{\mathcal I}\\
            &\stackrel{\text{(I.H)}}{\Leftrightarrow}&a\notin \overline{C_1}^{\overline\Imc_\Sigma}\text{ and }a\notin \overline{C_2}^{\overline\Imc_\Sigma}\\
            &\Leftrightarrow&a\notin (\overline{C_1}^{\overline\Imc_\Sigma}\cup \overline{C_2}^{\overline\Imc_\Sigma})\\
            &\Leftrightarrow&a\notin (\overline{C_1}\sqcup \overline{C_2})^{\overline\Imc_\Sigma} = \overline{(C_1\sqcap C_2)}^{\overline\Imc_\Sigma}
        \end{eqnarray*}
        \item $C = (C_1\sqcup C_2)$:
        \begin{eqnarray*}
            &&a\in (C_1\sqcup C_2)^{\mathcal I}\\
            &\Leftrightarrow&a\in C_1^{\mathcal I}\cup C_2^{\mathcal I}\\
            &\stackrel{\text{(I.H)}}{\Leftrightarrow}&a\notin \overline{C_1}^{\overline\Imc_\Sigma}\text{ or }a\notin \overline{C_2}^{\overline\Imc_\Sigma}\\
            &\Leftrightarrow&a\notin (\overline{C_1}^{\overline\Imc_\Sigma}\cap \overline{C_2}^{\overline\Imc_\Sigma})\\
            &\Leftrightarrow&a\notin (\overline{C_1}\sqcap \overline{C_2})^{\overline\Imc_\Sigma} = \overline{(C_1\sqcup C_2)}^{\overline\Imc_\Sigma}
        \end{eqnarray*}
        \item $C= \exists r.C_1$:
        \begin{eqnarray*}
            &&a\in(\exists r.C_1)^{\mathcal I}\\ 
            &\Leftrightarrow&\text{ there is } (a,b)\in r^{\mathcal I}\text{ with }b\in C_1^{\mathcal I}\\
            &\stackrel{\text{(I.H)}}{\Leftrightarrow}&\text{ there is } (a,b)\in r^{\overline\Imc_\Sigma}\text{ and }b\notin \overline{C_1}^{\overline\Imc_\Sigma}\\
            &\Leftrightarrow&a\notin (\forall r.\overline{C_1})^{\overline\Imc_\Sigma} = \overline{(\exists r. C_1)}^{\overline\Imc_\Sigma}
        \end{eqnarray*}
        \item $C=\forall r.C_1$:
            \begin{eqnarray*}
                &&a\in(\forall r. C_1)^{\mathcal I}\\
                &\Leftrightarrow& \text{for all }(a,b)\in r^\Imc:b\in C_1^\Imc \\
                &\stackrel{\text{(I.H)}}{\Leftrightarrow}&\text{for all }(a,b)\in r^{\mathcal I}: b\notin \overline{C_1}^{\overline\Imc_\Sigma}\\
                &\Leftrightarrow& \text{for all }(a,b)\in r^\Imc:b\in (\neg \overline C_1)^{\overline\Imc_\Sigma} \\
                 &\Leftrightarrow& a\in (\forall r.\neg \overline{C_1})^{\overline\Imc_\Sigma}= (\neg \exists r. \overline C_1)^{\overline\Imc_\Sigma} = (\neg \overline{\forall r.C_1})^{\overline\Imc_\Sigma} \\
                &\Leftrightarrow& a\notin \overline{\forall r. C_1}^{\overline\Imc_\Sigma}
            \end{eqnarray*}
    \end{enumerate}
\end{proof}

\fittingdual*

\begin{proof}
We show only the "only if"-direction; the "if"-direction is symmetric. Suppose $C$ fits $P,N$. By definition, $a\in C^{\mathcal I}$ for all $(\mathcal I,a)\in P$ and $b\notin C^{\mathcal J}$ for all $(\mathcal J,b)\in N$. By Lemma~\ref{lem_dual}, $a\notin \overline C^{\overline\Imc_\Sigma}$ and $b\in \overline C^{\overline\Jmc_\Sigma}$ for all $(\mathcal I,a)\in P$ and $(\mathcal J,b)\in N$. Thus, $\overline C$ fits $\overline N_\Sigma,\overline P_\Sigma$. 
\end{proof}

\section{Proofs for Section~\ref{sec:complexity}}

\proplower*

We give here the remaining details of the proof for Proposition~\ref{prop:technplower} sketched in the main part. We start with giving a precise definition of interpretations~$\Imc,\Jmc$.

Recall that $(\{S_1,\ldots,S_m\},k)$ is the hitting set instance we consider, and that we assume that $\bigcup_i S_i=\{1,\ldots,n\}$. We first define an interpretation $\mathcal J'$ that encodes only the sets $S_1,\ldots,S_m$ and then extend it to obtain \Imc and \Jmc. As described in the main part, $\mathcal J'$ contains an $r$-path of length $n$ for each set $S_j$, with suitable detours via $s$. The elements on the $r$-paths are called $b_{j,0},\ldots,b_{j,n}$ for each set $S_j$, and the extra elements are called $b_{j,i}'$. Recall that there is an additional sink element $c$. The interpretation $\mathcal J'$ is given as follows. 
\begin{align*}
  \Delta^{\mathcal J'} ={}& \{b_{j,i}\mid 1\le j \le m, 0\le i \le n\}\cup{}\\
  & \{b_{j,i}'\mid 1\le j \le m, 0\le i \le n, i \notin S_j\}\cup \{c\} \\
A^{\mathcal J'} ={}& \{b_{j,n}\mid 1\le j\le m\}\\
           r^{\mathcal J'} ={}&\{(b_{j,i-1}, b_{j,i})\mid 1\le j\le m, 1\le i \le n\}\cup\{(c,c)\}\cup{}\\
           & \{(b_{j,i}', c)\mid 1\leq j\leq m,1\leq i\leq n, i\notin S_j\}\cup{}\\
           & \{(b_{j,n}, c)\mid 1\leq j\leq m\}\\
           s^{\mathcal J'} ={}& \{(b_{j,i-1}, b'_{j,i}),(b'_{j,i},b_{j,i})\mid 1\le j\le m, 1\le i \le n, i\notin S_j\}\cup{}\\
           & \{(b_{j,i-1}, c)\mid 1\le j\le m, 1\le i \le n, i\in S_j\} \cup\{(c,c)\}\cup{}\\
           & \{(b_{j,n}, c)\mid 1\leq j\leq m\}
\end{align*}
We construct \Imc by extending $\mathcal J'$ as follows. 
\begin{align*}
\Delta^{\mathcal I} ={}& \{a\}\cup \{a_{i},a_i'\mid 0\le i\le n\} \cup \Delta^{\mathcal J'}\\
            A^{\mathcal I} ={}& \{a_{n}\}\cup A^{\mathcal J'}\\
            r^{\mathcal I} ={}& \{(a,a_0)\}\cup\{(a_{i-1},a_{i})\mid 1\le i\le n\}\cup r^{\mathcal J'}\cup{} \\
            & \{ (a, b_{j, 0}) \mid 1 \leq j \leq m\} \cup {} \\
            & \{(a_i',c)\mid 1\leq i\leq n\}\cup \{(a_n,c)\}\\
            s^{\mathcal I} ={}& \{(a_{i-1},a'_{i}), (a_i',a_i)\mid 1\le i\le n\}\cup \{(a_n,c)\}\cup s^{\mathcal J'} 
            \end{align*}
The interpretation $\mathcal J$ is the extension of $\mathcal J'$ with an element $b$ that connects to the starting points of the components representing the sets in $S$. It is formally defined as follows.
\begin{align*}
\Delta^{\mathcal J} &= \{b\}\cup \Delta^{\mathcal J'}\\            
A^{\mathcal J} &= A^{\mathcal J'}\\
r^{\mathcal J} &=\{(b,b_{j,0})\mid 1\leq j\leq m\}\cup r^{\mathcal J'}\\
s^{\mathcal J} &= s^{\mathcal J'}
\end{align*}

We define $P=\{(\Imc,a)\}$, $N=\{(\Jmc,b)\}$, and $k'=n+k+2$. To show Equivalence~\eqref{eq:correctness} from the main part, it suffices to show that the following are equivalent. 
\begin{enumerate}[label=(\roman*)]

  \item $S$ has a hitting set of size at most $k$; 
  
  \item there is an $\Lmc(\{\exists\})$ concept fitting $P,N$ of size at most $k'$; 
  
  \item there is an $\ALC$ concept fitting $P,N$ of size at most $k'$.
\end{enumerate}

 For implication (i) $\Rightarrow$ (ii), consider a hitting set $H$ and the concept $D=\exists r.C_n$ defined based on $H$. Based on the construction of $C_n$, it is easy to verify that $a\in D^\Imc$. To see that $b\notin D^\Jmc$, suppose, with the aim of showing a contradiction, that $b_{j,0}\in C_n^\Jmc$, for some $j$ with $1\leq j\leq m$. Based on the definition of \Jmc and the fact that each $C_i$ needs a reachable point satisfying $A$, we can verify using induction that:
 \begin{equation}\label{eq:np-hardness-ind}
  b_{j,\ell}\in C_{n-\ell}^\Jmc \text{ for all } \ell\in\{0,\ldots,n\}.
 \end{equation}
Consider any element $i\in S_j\cap H$ and consider $b_{j,i-1}$, which by Equation~\eqref{eq:np-hardness-ind} satisfies $b_{j,i-1}\in C_{n-i+1}^\Jmc$. By definition of $C_{n-i+1}$ and $i\in H$, $C_{n-i+1}$ is of shape $\exists s.\exists s.C_{n-i}$. By definition of $\Jmc$, the only $s$-successor of $b_{j,i-1}$ is $c$ which does not satisfy $C_{n-i}$, a contradiction.
   
Implication (ii) $\Rightarrow$ (iii) is trivial. 

To prove (iii) $\Rightarrow$ (i), let $C_0$ be an $\alc$ concept of size $\lVert C_0\rVert\leq k'$ that fits~$P,N$. For now, we assume $C_0$ to be in negation normal form, that is, negation occurs only in front of concept names. We argue that this is without loss of generality later. We will show that $C_0$ can be gradually transformed into an $\Lmc(\{\exists\})$ concept $C$
of size $\lVert C\rVert\le k'$ that fits $P,N$.

\bigskip
We start restricting the overall shape. 

\begin{claim}
There is an \ALC concept $C$ of the form $\exists r. D$ of size at most $k'$ that fits $P,N$.
\end{claim}

\noindent\textit{Proof of the claim.} We first observe that $C_0$ is not the concept name $A$ as $a\notin A^{\mathcal I}$. If $C_0$ is a disjunction or conjunction, we can choose one of the disjuncts or conjuncts, respectively, as a smaller fitting concept $C$. (This is due to the fact that there are only single positive and negative examples, respectively.) Because both $a$ and $b$ have only $r$-successors, $C$ must have the form $\exists r. D$ or $\forall r.D$. We can rule out the latter as $a\in (\forall r.D)^{\mathcal I}$ would imply $b\in(\forall r.D)^{\mathcal J}$. \qed

\medskip
Next, we will eliminate universal restrictions. 
\begin{claim}
There is an $\mathcal L(\{ \exists,\sqcap,\sqcup,\neg \})$ concept of shape $\exists r.D$ and size at most $k'$ that fits $P,N$.
\end{claim}

\noindent\textit{Proof of the claim.} Take the concept $C$ fitting $P,N$ resulting from the previous claim. Since $C$ is of the form $\exists r.D$, the definitions of $\Imc,\Jmc$ implies that: 
\begin{itemize}
    \item $a_0\in D^\Imc$, and 
    \item $b_{j,0}\notin D^\Jmc$ for all $j\in \{1,\ldots,m\}$.
\end{itemize}
Consider the sub-interpretations of $\Imc$ rooted in $a_0$ and the sub-interpretations of \Jmc rooted at $b_{j,0}$ for all $j$. It is readily seen that in all these sub-interpretations, the roles $r$ and $s$ are interpreted as total and functional relations. It is well known that in such interpretations the extensions of $\exists r.E$ and $\forall r.E$ coincide for all \ALC concepts $E$. Hence, if we replace in $D$ every sub-concept $\forall r.E$ by $\exists r.E$, the resulting concept $D'$ still satisfies $a_0\in (D')^\Imc$ and 
$b_{j,0}\notin (D')^\Imc$ for all $j\in \{1,\ldots,m\}$. The concept $\exists r.D'$ is then as required by the claim.\qed
\medskip
Next, we will eliminate disjunctions. 
   \begin{claim}
       There is a $\mathcal L(\{ \exists,\sqcap,\neg \})$ concept of shape $\exists r.D$ and of size at most $k'$ that fits $P,N$.
   \end{claim}

   \noindent\textit{Proof of the claim.} Let $C=\exists r.D$ be a fitting $\mathcal L(\{\exists,\sqcap,\sqcup,\neg\})$ concept that exists due to the previous claim. Using distributivity laws, we can transform the concept $D$ into an equivalent concept $C_1\sqcup\dots\sqcup C_l$ such that each $C_i$ is an $\mathcal L(\{\exists,\sqcap,\neg\})$ concept. One can easily verify that $\lVert C_i\rVert\le \lVert D\rVert$ for each $i\in \{1,\dots,l\}$ after the transformation. For at least one $i\in\{1\dots,l\}$ we must have $a'\in C_i^{\mathcal I}$ for some $r$-successor $a'\in\Delta^{\mathcal I}$ of $a$. We choose the concept $\exists r.C_i$ as a fitting $\mathcal L(\{\exists,\sqcap,\neg \})$ concept which is at most as large as $C$.\qed

   \medskip
   
   Finally, we proceed by eliminating conjunctions and negations. 

   \begin{claim}
        There is a $\mathcal L(\{\exists \})$ concept $C$ of shape $C=\exists r.D$ of size at most $k'$ that fits $P,N$.
   \end{claim}

\noindent\textit{Proof of the claim.} Let $C=\exists r.D$ be the $\Lmc(\{\exists,\sqcap,\neg\})$ concept that exists due to the previous claim.
We need some auxiliary notation. For a finite sequence $w=r_1\cdots r_m\in\{r,s\}^*$, we abbreviate with $\exists w.C$ the concept $\exists r_1.\ldots.\exists r_m.C$. Moreover, for some interpretation $\Imc_0$ and $a_0\in \Delta^{\Imc_0}$, we let $R_w(\Imc_0,a_0)$ denote the set of all $b$ such that $(a_0,b)\in r_1^{\Imc_0}\circ\ldots\circ r_m^{\Imc_0}$, where $\circ$ denotes the composition of binary relations.

We will transform $C$ inductively into concepts of the shape $C_1=\exists w_1.D_1,C_2=\exists w_2.D_2,\ldots$ such that each $C_i$ fits $P,N$, is of size at most $k'$, and 
\begin{enumerate}

    \item $R_{w_i}(\Imc,a)\cap \{a_0,a_0',\ldots,a_n,a_n'\}$ contains a single element and this element satisfies $D_i$;

    \item $R_{w_i}(\Jmc,b) \neq \emptyset$ and no element from $R_{w_i}(\Jmc,b)$ satisfies $D_i$.
    
\end{enumerate}

In the inductive base, we set $w_1=r$ and $D_1=D$.

Clearly, Properties~1 and~2 are satisfied by this choice.

\smallskip For the inductive step, take any $C_i=\exists w_i.D_i$. We distinguish cases on the shape of $D_i$. The easy cases are the following: 
\begin{itemize}

    \item If $D_i=A$, we are done.
    
    \item The case $D_i=\top$ is impossible due to Property~2 above.
    
    \item The case $D_i=\neg A$ is also impossible due to Property~2 above.
    
    \item If $D_i=\exists t.D_i'$ for $t\in \{r,s\}$, we can just extend $w_i$ with $t$, that is, $w_{i+1}=w_it$ and $D_{i+1}=D_i'$.

\end{itemize}

The final case requires a bit more care. Let $D_i=E_1\sqcap \ldots\sqcap E_m$ be a non-empty conjunction.  If some conjunct is $\top$, we can simply remove it. If some conjunct is $\neg A$, we can remove it due to Property~2. If there are two conjuncts $\exists r.F$, $\exists r.G$, then we can replace them with $\exists r.(F\sqcap G)$, due to functionality, and the same with $s$. 

Suppose now that some conjunct in $D_i$ is $A$. Then the element witnessing Property~1 has to be $a_n$. This means that for every other conjunct $\exists t.F$ in $D_i$, $t\in\{r,s\}$, $F$ is satisfied in $c$. However, every $\Lmc(\{\exists,\sqcap,\neg\})$ concept in negation normal form satisfied in $c$ is satisfied in every element in $R_{w_i}(\Jmc,b)$. Hence, we can simply set $D_i=A$ and are done. 

Hence, we can assume that $D_i$ is actually of shape $\exists r.F_1\sqcap \exists s.F_2$. 

We distinguish cases on the shape of the element $d$ from Property~1.
\begin{itemize}
    \item If $d=a_\ell$, for some $\ell$, then we can replace $D_i$ with $\exists s.\exists s.F_1$. This is clearly satisfied in $d$. Moreover, in this case, elements in $R_{w_i}(\mathcal J,b)$ are either $c$ or $b_{j,k}$, and all of them do not satisfy $\exists s.\exists s.F_1$, by the assumption that they do not satisfy $\exists r.F_1\sqcap \exists s.F_2$. Thus Properties~1 and~2 are preserved. 
    
    \item If $d=a_\ell'$, for some $\ell$, then we can drop $\exists r.F_1$ since $F_1$ has to be satisfied in $c$ and every $\Lmc(\exists,\sqcap,\neg)$ concept in negation normal form satisfied in $c$ is satisfied in every element in $R_{w_i}(\Jmc,b)$. This preserves Properties~1 and~2.
    
\end{itemize}
This finishes the proof of the Claim.\qed

\bigskip
   
Let now be $C$ a fitting $\Lmc(\{\exists\})$ concept that fits $P,N$ and is of size at most $k'$. We make two observations on the shape of $C$, using the notation of the proof of the previous claim.
\begin{itemize}

    \item $C$ is of shape $\exists r.\exists w.A$ for some $w$: the only other option for the innermost concept would be $\top$, but $\exists r.\exists w.\top$ is satisfied in $b$. 

    \item $w\in (r+ss)^+$ since this is the only way to "reach" $a_n$ in the positive example. Moreover, due to the size bound $\lVert C \rVert \leq k'=n+k+2$, the pattern $ss$ occurs at most $k$ times in $w$. 
    
\end{itemize}
Due to the second item above we can write $w=w_1\ldots w_n$ with $w_i\in\{r,ss\}$ for all $i$ and such that $w_i=ss$ at most $k$ times. We read off a hitting set $H$ of size at most $k$ from $C$ by taking
\[H = \{i\in\{1,\ldots,n\}\mid w_i=ss\}\]
We have to show $H\cap S_i\neq \emptyset$, for every $j$. Suppose, with the goal of deriving a contradiction, that $H\cap S_j=\emptyset$. Then, by the construction of $\Jmc$, we have $b_{j,0}\in (\exists w.A)^\Jmc$. Hence, $b\in C^\Jmc$ and $(\Jmc,b)$ is not a negative example, contradiction.

\bigskip
We have yet to justify our initial assumption that a fitting concept is given in negation normal form. For that, consider any smallest fitting concept $C_0$ (not necessarily in negation normal form), that is $\lVert C\rVert \ge \lVert C_0\rVert$ for all fitting concepts~$C$. Let $C_0'$ be the negation normal form of $C_0$. We show that we either can obtain a fitting $\mathcal L(\{\exists\})$ concept as outlined above, or that we can construct a smaller fitting concept than $C_0$ thus reaching a contradiction. We cannot have $\lVert C_0'\rVert <\lVert C_0\rVert$ as $C_0$ is a smallest fitting concept. In the case of $\lVert C_0'\rVert=\lVert C_0\rVert$ we can obtain a fitting $\mathcal L(\{\exists\})$ concept from $C_0'$ as argued above. Now, assume $\lVert C_0'\rVert > \lVert C_0\rVert$. In this case, $C_0$ must contain at least one negation and $C_0'$ must contain at least two negations. A closer look at the above proof yields that from such a concept $C_0'$ we can obtain a fitting concept $C'$ that is smaller than $C_0$, i.e. $\lVert C'\rVert< \lVert C_0\rVert$. In all transformations on the original fitting concept we either directly choose a subconcept of the original fitting concept or we remove negated concept names to obtain a fitting concept of the desired form. Thus, when an $\mathcal L(\{\exists,\sqcap,\neg \})$ concept is reached, negated concept names may have been removed but none have been added. The remaining negated concept names may then be removed as argued above. In conclusion, we construct a fitting concept $C$ by eliminating all negated concept names while possibly eliminating other subconcepts as well. Since $C_0$ contains at least one negation and $C$ contains none, we have obtained a fitting concept smaller than $C_0$ contradicting our assumption that $C_0$ is a smallest fitting concept. It follows that the negation normal form of any smallest fitting concept for these examples is of the same size as the smallest fitting concept, and thus we obtain a fitting concept $\mathcal L(\{\exists\})$ by assuming a smallest fitting concept to be given in NNF. \qed

\bigskip
Corollary~\ref{cor:reduction} implies that Proposition~\ref{prop:technplower} also holds for all fragments $\Lmc(O)$ with $\{\forall\}\subseteq O \subseteq \oall$. Thus, it also completes the proof of Theorem~\ref{thm:main1}.

\section{Proofs for Section~\ref{sec:generalization}}

\begin{lemma}\label{lem:fittingforall}
The fitting problem for $\Lmc(\{\forall\})$, $\Lmc(\{\forall,\sqcap\})$, and $\Lmc(\{\forall,\exists,\sqcap\})$ is \NPclass-hard. 
\end{lemma}

\begin{proof}
The proof is inspired by the proof of Theorem~3 in~\cite{DBLP:journals/ml/CohenH94}.

    We show this by reduction from SAT. Let $\varphi$ be a propositional formula in conjunctive normal form with $m$ clauses $C_1,\ldots,C_m$ over $n$ variables $x_1,\ldots,x_n$. For $w=w_1\cdots w_n\in\{0,1\}^n$ and some formula $\psi$, we write $w\models \psi$ if the assignment $\alpha$ defined by setting $\alpha(x_i)=w_i$, for all $i$, satisfies $\psi$.
    
    From each clause $C_i$, we can construct in time polynomial in the size of $\varphi$ a deterministic finite automaton $\Amc_i$ satisfying:
    \begin{itemize}
        \item $L(\Amc_i)\subseteq \{0,1\}^n$, and
        \item $w\in L(\Amc_i)$ if and only if $w\models C_i$, for all $w\in \{0,1\}^n$.
    \end{itemize}
    We transform each $\Amc_i=(Q_i,\delta_i,q_0^i,F_i)$ into an interpretation $\Imc_i$ using two role names $r_0,r_1$ and one concept name $F$ as follows: 
    \begin{align*}
        \Delta^{\Imc_i} & = Q_i \\
        F^{\Imc_i} & = F \\
        r_0^{\Imc_i} & = \{(q,q')\mid \delta(q,0)=q'\} \\
        r_1^{\Imc_i} & = \{(q,q')\mid \delta(q,1)=q'\}
    \end{align*}
    Let $\Jmc$ be the interpretation with domain $\Delta^\Jmc = \{a\}$ and $F^\Jmc = \emptyset$ and $r_0^\Jmc=r_1^\Jmc =\{(a,a)\}$. 

    \begin{claim} The following are equivalent for $P=\{(\Imc_i,q_0^i)\mid 1\leq i\leq m\}$, $N=\{(\Jmc,a)\}$: 
    \begin{enumerate}[label=(\roman*)]
    
        \item $\varphi$ is satisfiable;
        
        \item there is an $\Lmc(\{\forall\})$ concept fitting $P,N$;

         \item there is an $\Lmc(\{\forall,\exists,\sqcap\})$ concept fitting $P,N$.
        
    \end{enumerate}
    \end{claim}
    
    \noindent\textit{Proof of the claim.} To show, (i)$\Rightarrow$(ii), let $\alpha$ be a satisfying assignment of $\varphi$. Based on the definition of the $\Imc_i$ and $\Jmc$, it is routine to verify that the concept
    \[\forall r_{\alpha(x_1)}\ldots\forall r_{\alpha(x_n)}.F\]
    fits $P,N$.

    (ii)$\Rightarrow$(iii) is trivial. 

    For (iii)$\Rightarrow$(i), suppose $D$ is an $\Lmc(\{\forall,\exists,\sqcap\})$ concept that fits $P,N$. Let $D'\in\Lmc(\{\forall,\sqcap\})$ be obtained from $D$ by replacing every existential restriction in $D$ by a universal one (of course over the same role name). Since in all involved examples, $r_0$ and $r_1$ are interpreted as total functions, $D'$ also fits $P,N$. 
   
    Based on the facts that (a) there is only one negative example and~(b) roles are interpreted as functions, we can show that conjunction is not necessary in the sense that we 
    can obtain from $D'$ an $\Lmc(\{\forall\})$ concept $D''$ that fits $P,N$. By construction of the $\Imc_i$ and \Jmc, this concept has to take the shape
    \[\forall r_{i_1}\ldots\forall r_{i_n}.F\]
    It is now routine to verify that $w\models\varphi$ for $w=i_1\cdots i_n$.

This finishes the proof of the claim and, in fact, of the lemma.
    \qed
\end{proof}

\begin{corollary} There is no efficient PAC learning algorithm for $\Lmc(\{\forall\})$, $\Lmc(\{\forall,\sqcap\})$, and $\Lmc(\{\forall,\exists,\sqcap\})$, unless $\RPclass=\NPclass$. 
\end{corollary}

\begin{proof} Let $\Cmf$ be a class of examples. We say that \Lmc has the \emph{polynomial size fitting property over \Cmf} if any for every sample $P,N$ from this class \Cmf: if there is an \Lmc concept fitting $P,N$, there is one of polynomial size in the size of $P, N$. We also say that $\Lmc$ is \emph{polynomial time evaluable} if given an example $(\Imc,a)$ and $C\in \Lmc$, we can determine in polynomial time whether $a\in C^\Imc$. It is well-known that all fragments \Lmc of $\ALC$ are polynomial time evaluable. 

A classic theorem by Pitt and Valiant states that if \Lmc has the polynomial size fitting property (possibly over a class \Cmf) and is polynomial time evaluable (over \Cmf), then efficient PAC learnability implies that the fitting problem (over \Cmf) is in RP~\cite{PittValiant88}. 

Let us denote with $\Cmf_0$ the class of all interpretations obtained from the
construction in the hardness proof in Lemma~\ref{lem:fittingforall}. It is 
not difficult to see that $\Lmc(\{\forall\})$, $\Lmc(\{\forall,\sqcap\})$, and
$\Lmc(\{\forall,\exists,\sqcap\})$ have the polynomial fitting property over the
class $\Cmf_0$. 

To show the statement from the corollary, let us assume that there is an efficient PAC learning algorithm for 
\Lmc from $\Lmc(\{\forall\})$, $\Lmc(\{\forall,\sqcap\})$, and
$\Lmc(\{\forall,\exists,\sqcap\})$.
By Pitt's and Valiant's theorem, the fitting problem for $\Lmc$ over $\Cmf_0$ is in $\RPclass$. The proof of Lemma~\ref{lem:fittingforall} implies that the fitting problem for \Lmc is $\NPclass$-hard already over that class $\Cmf_0$. Altogether, this implies $\RPclass=\NPclass$.
\qed
\end{proof}

\lemnotsampleefficientdepth* 

\begin{proof}
  It remains to argue that we can choose $n$ large enough such that the probability that a sample of size $m(1/\delta,1/\varepsilon,3,2n+1,n+2)$ does not contain both $(\Imc_w,a_w)$ and $(\Jmc_w,a_w)$ for some $w\in \{r,s\}^n$ is at least~$1-\delta$. Since we know that $m$ is a polynomial, there is a polynomial $p(n)$ such that $p(n)=m(1/\delta,1/\varepsilon,3,2n+1,n+2)$ for all $n$.
 
Let $\mathbb P_n'$ be the uniform distribution over a set of
$2^n$ elements. We observe that the probability of
seeing in a sample of $\mathbb P$ of size $p(n)$ not both $(\Imc_w,a_w)$ and $(\Jmc_w,a_w)$ for some $w\in \{r,s\}^n$ is
greater than the probability of seeing in a sample of size
$p(n)$ over $\mathbb P_n'$ no element twice. 

The probability of sampling $\ell$ different elements from a uniform distribution over an $N$-element set is computed as the following fraction: 
\begin{equation}\label{eq:differentsampling}
\frac{\prod_{i=0}^{\ell-1}(N-i)}{N^\ell}=\frac{N!}{N^\ell\cdot(N-\ell)!}
\end{equation}
Hence, it suffices to show that the limit of the fraction in the right of Equation~\eqref{eq:differentsampling} when we put $N=2^n$ and $\ell=p(n)$ and let $n$ go towards $\infty$ is $1$. This is precisely the statement of~\cite[Lemma~6]{IJCAI23arxiv}.
    \qed
\end{proof}

\lemmostgeneralexistscup*

\begin{proof}
We give the missing argument that $C_0=A\sqcup\bigsqcup_{(\Imc_w,a_w)\notin N} C_w$ is the unique most general $\Lmc(\{\exists,\sqcup\})$ concept fitting $P,N$. Clearly, $C_0$ fits $P,N$. So let $D$ be any $\Lmc(\{\exists,\sqcup\})$ that fits $P,N$. We have to show that $D\sqsubseteq C_0$. 

Consider the syntax trees of some $\Lmc(\{\exists,\sqcup\})$ concept. A \emph{path in the syntax tree} is any sequence $r_1\ldots r_kA$ that can be read by starting at the root of the syntax tree, reading only the role name for $\exists r$-nodes, skipping $\sqcup$-nodes, and finally reading the label of the leaf (which is either a concept name or $\top$). 

It can be verified that $D\sqsubseteq C_0$ holds if and only if all paths in the syntax tree of $D$ are paths in the syntax tree of $C_0$. So consider a path $r_1\ldots r_kA$ in the syntax tree of~$D$. We call $k$ the length of the path. Since $(\Jmc,b_1)\in N$ and $D$ fits $P,N$, the path cannot have length $1,\ldots,n-1,n+1$ (otherwise, it would incorrectly label $(\Jmc,b_1)$). Hence, it can only have length $0$ or length $n$. If it is of length $0$, then it has to have the shape $A$ and this path is also contained in $C_0$, by definition. Otherwise, it is of shape $wA$ with $w\in\{r,s\}^n$. But since $D$ fits $P,N$, we know that $(\Imc_w,a_w)\notin N$. By definition, $C_0$ contains the path $wA$ in this case.
\qed
\end{proof}

It remains to cover the cases $\Lmc(\{\exists,\forall,\sqcap\})$ and 
$\Lmc(\{\exists,\sqcup,\sqcap\})$ as announced in the body of the paper. Let us state the lemmas formally. 

\begin{lemma}\label{lem:mostgeneralelu}
    Let $\Amc$ be a fitting algorithm for $\Lmc(\{\exists,\sqcup,\sqcap\})$ that always returns always a most general fitting if one exists. Then $\Amc$ is not a sample-efficient PAC learning algorithm.
\end{lemma}

\begin{lemma}\label{lem:mostgeneralflzero}
    Let $\Amc$ be a fitting algorithm for $\Lmc(\{\exists,\forall,\sqcap\})$ that always returns always a most general fitting if one exists. Then $\Amc$ is not a sample-efficient PAC learning algorithm. 
\end{lemma}

The proofs of both Lemmas~\ref{lem:mostgeneralelu} and and~\ref{lem:mostgeneralflzero} are (non-trivial) variations of the proof for the same fact for $\EL=\Lmc(\{\exists,\sqcap\})$~\cite[Theorem~3/Theorem~7]{IJCAI23arxiv}. We require some additional notation and facts which we introduce next. 

Every \EL concept $C$ may be viewed as tree-shaped interpretation $\Imc_C$ in an obvious way,
e.g.\ $C=\exists r . \exists s . A$ results in a path of length $2$ which ends in a node satisfying $A$ and whose edges are labeled with $r$ and $s$ in this order. We denote with $a_C$ the root node of this tree. Let $\Imc_1,\Imc_2$
be interpretations and $\Sigma$ a signature. A \emph{$\Sigma$-simulation}
from $\Imc_1$ to $\Imc_2$ is a relation
$S \subseteq \Delta^{\Imc_1} \times \Delta^{\Imc_2}$ such that for all $(d_1,d_2)\in S$:
  \begin{enumerate}

  \item if $d_1\in A^{\Imc_1}$ with $A \in \Sigma$, then
    $d_2 \in A^{\Imc_2}$;

  \item if $(d_1, e_1) \in r^{\Imc_1}$ with $r \in \Sigma$, there is
    $(d_2,e_2) \in r^{\Imc_2}$ such that $(e_1, e_2) \in S$.

  \end{enumerate}
  
  For $d_1\in\Delta^{\Imc_1}$ and $d_2\in \Delta^{\Imc_2}$, we
  write $(\Imc_1,d_1)\preceq_\Sigma (\Imc_2,d_2)$ if there is a
  $\Sigma$-simulation $S$ from $\Imc_1$ to $\Imc_2$ with
  $(d_1,d_2)\in S$. 

Simulations are important since on the one hand they conveniently characterize evaluation and subsumption in \EL, as defined in the following two well-known statements. 

\begin{lemma}\label{lem:eval-simulation}
For all $\EL$-concepts $C$, interpretations $\Imc$, and $d\in C^{\Imc}$, we have $d\in C^{\Imc}$ iff $(\Imc_C,a_C)\preceq (\Imc,d)$.
\end{lemma}

\begin{lemma}\label{lem:subsumption-simulation}
For all $\EL$-concepts $C,D$, we have $C\sqsubseteq D$ iff $(\Imc_D,a_D)\preceq (\Imc_C,a_C)$.
\end{lemma}

We next introduce the notion of \emph{simulation duals.}
\begin{definition}
Let $(\Imc,a)$ be an example and $\Sigma$ a signature.  A set
$M$ of examples is a \emph{$\Sigma$-simulation dual} of
$(\Imc,a)$ if for all examples $(\Imc', a')$, the following
holds:
\[(\Imc, a) \preceq_\Sigma (\Imc', a') 
  \quad\text{ iff }\quad (\Imc', a') \not\preceq_\Sigma (\Imc'', a'') 
  \text{ for all } (\Imc'',a'') \in M.  \]
\end{definition}
We will rely on the following theorem from~\cite{DBLP:conf/ijcai/CateFJL23}.

\begin{theorem}[{\cite[Theorem~3]{DBLP:conf/ijcai/CateFJL23}}]\label{thm:duals} Given an \EL concept
  $C$ and a signature
  $\Sigma$, a $\Sigma$-simulation dual $M$ of $(\Imc_C,a_C)$
  of size
 $\lVert M \rVert \leq 3\cdot |\Sigma| \cdot \lVert C \rVert^2$
  can be
  computed in polynomial time.
  Moreover, if $\Imc_C$, when viewed as a set, contains only a single $\Sigma$-assertion 
  that mentions~$a_C$, then $M$ is a singleton.
\end{theorem}

We are now ready to move on to the proofs of Lemmas~\ref{lem:mostgeneralelu} and~\ref{lem:mostgeneralflzero}. Both proofs use  signature $\Sigma=\{A,B,r\}$ and a family of concept names $(C_n)_{n\geq 0}$ inductively defined as follows:
\begin{align*}
C_0 & = \top \\
C_{i+1} & = \exists r.(A\sqcap B\sqcap C_i), \text{ for all $i\geq 0$.}
\end{align*}
We give now first the (easier) proof of Lemma~\ref{lem:mostgeneralelu}, which is rather close to the proof of~\cite[Theorem~3/7]{IJCAI23arxiv}.
\begin{proof}[of Lemma~\ref{lem:mostgeneralelu}]
Let \Amc be a fitting algorithm for $\Lmc(\{\exists,\sqcup,\sqcap\})$ that always returns a most general fitting if one exists, and additionally is a sample-efficient PAC learning algorithm with associated polynomial $m$. Set $\delta=\varepsilon=\frac 12$ and $n$ large enough so that 
\[2^n>4m(1/\delta,1/\varepsilon,3,p(n),3n)\]
for some polynomial $p$ to be determined below. 

As signature we use $\Sigma=\{A,B,r\}$.
As target concept $C_T$ we use $C_T:=C_n$. 
We use only negative examples. To construct them, let us first define a set of \EL concepts $S_n$, inductively as follows:
\begin{align*}
S_0 & = \{\top\} \\
S_{i+1} & = \{\exists r.(\alpha\sqcap C)\mid C\in S_i,\alpha\in \{A,B\}\}, \text{ for $i\geq 0$}
\end{align*}
Clearly, $S_n$ contains $2^n$ \EL concepts, and $C_T\sqsubseteq C$ and $\lVert C \rVert \leq 3n$, for every $C\in S_n$. Consider the example $(\Imc_C,a_C)$ with $C\in S_n$. Note that $\Imc_C$ contains a single assertion that mentions $a_C$. By Theorem~\ref{thm:duals}, $(\Imc_C,a_c)$ has a singleton $\Sigma$-simulation dual $\{(\Imc_C'',a_c')\}$ with $\lVert \Imc_C''\rVert \leq 3\cdot|\Sigma|\cdot 3n$. These duals are negative examples for $C_T$ since $C_T\sqsubseteq C$, using Lemma~\ref{lem:subsumption-simulation} and the definition of duals. 

Let the probability distribution $\mathbb P$ assign probability probability $1/{2^{n}}$ to each $(\Imc_C',a_C')$, $C\in S_n$.

\medskip
\noindent\textit{Claim~1.}
For each $C\in S_n$, $C$ is the most general $\Lmc(\{\exists,\sqcup,\sqcap\})$ concept that fits the negative example $(\Imc_C',a'_C)$.

\medskip
\noindent\textit{Proof Claim~1.} The very same statement with $\EL$ in place of $\Lmc(\{\exists,\sqcup,\sqcap\}$ is stated in the original proof, and we rely on this fact. Let $D$ be an $\Lmc(\{\exists,\sqcup,\sqcap\}$ concept that fits the negative example $(\Imc'_C,a_C')$. We show that $D\sqsubseteq C$, hence $C$ is also the most general $\Lmc(\{\exists,\sqcup,\sqcap\})$ concept that fits the negative example $(\Imc'_C,a_C')$. We can equivalently write $D$ as a disjunction of \EL-concepts $D\equiv D_1\sqcup\ldots\sqcup D_n$. Since $D$ fits negative example $(\Imc'_C,d_C')$, Lemma~\ref{lem:eval-simulation} yields $(\Imc_{D_i},a_{D_i})\not\preceq (\Imc'_C,d_C')$, for all $i$. By definition of duals, we get $(\Imc_C,a_C)\preceq (\Imc_{D_i},a_{D_i})$, for all $i$. Lemma~\ref{lem:subsumption-simulation} yields $ D_i\sqsubseteq C$, for all $i$, and hence $D\sqsubseteq C$.\hfill$\dashv$

\medskip Let now $\emptyset,N$ be a sample of $m(1/\delta,1/\varepsilon,3,p(n),3n)$ examples from $\mathbb P$ (recall that there are no positive examples).

\medskip\noindent\textit{Claim 2.} $C_H=\bigsqcap_{(\Imc_C',a_C')\in N} C$ is the most general $\Lmc(\{\exists,\forall,\sqcap\})$ concept fitting $\emptyset,N$.

\medskip\noindent\textit{Proof of Claim~2.} We first observe that $C_H=\bigsqcap_{(\Imc_C',a_C')\in N} C$ fits $\emptyset,N$, since, by the definition of duals, $(\Imc_C',a_C')$ is a negative example for each such $C$.

We now verify that it is the most general one. Let $D$ be any $\Lmc(\{\exists,\forall,\sqcap\})$ concept that fits $\emptyset,N$. By Claim~1, $D\sqsubseteq C$, for each $(\Imc_{C}',a_C')\in N$. Hence, $D\sqsubseteq C_H$.
\hfill$\dashv$

\medskip Hence, (a concept equivalent to) $C_H$ is output by the algorithm. We claim that $C_H$ labels all negative examples not in the sample wrong. Indeed, $C_H\not\sqsubseteq C'$ for any $C'$ with $(\Imc''_{C'},a_C')\notin N$ (not difficult to see and shown in the original proof). By definition of duals, $(\Imc_{C'}'',a_C')$ is labeled positively by~$C_H$. Since there are $2^{n}$ negative examples, each with probability $1/{2^n}$, the choice of $n$ implies that the error is greater than $\varepsilon=1/2$.
\qed 
\end{proof}

We move to the proof of Lemma~\ref{lem:mostgeneralflzero}. It is a variant of the above proof, but needs much more involved constructions.

\begin{proof}[of Lemma~\ref{lem:mostgeneralflzero}]
Let \Amc be a fitting algorithm for $\Lmc(\{\exists,\forall,\sqcap\})$ that always returns a most general fitting if one exists, and additionally is a sample-efficient PAC learning algorithm with associated polynomial $m$. Set $\delta=\varepsilon=\frac 12$ and $n$ large enough so that 
\[2^n>4m(\frac{1}{\delta},\frac{1}{\varepsilon},3,p(n),3n)\]
for some polynomial $p$ to be determined below. 

As signature we use $\Sigma=\{A,B,r\}$.
As target concept $C_T$ we use $C_T:=C_n$ which is defined inductively as follows:
\begin{align*}
C_0 & = \top \\
C_{i+1} & = \exists r.(A\sqcap B\sqcap C_i), \text{ for all $i\geq 0$.}
\end{align*}
We use a single positive example $(\Imc,a_0)$ and $2^n$ negative examples. The mentioned proof of~\cite[Theorem~3/7]{IJCAI23arxiv} uses the same negative examples, but no positive example. Intuitively, we use the additional positive example to ``switch off'' universal quantification $\forall$.

The positive example $(\Imc, a_0)$ is defined as follows:
\begin{align*}
    \Delta^\Imc & =\{a_i,b_i,c_i\mid 0\leq i\leq n\}\\
    A^\Imc & =B^\Imc = \{a_1,\ldots,a_n\} \\
    r^\Imc &= \{(a_i,a_{i+1}),(a_i,b_i),(a_i,c_i),(c_i,c_i)\mid 0\leq i< n\}
\end{align*}
Intuitively, $\Imc$ is an $r$-path of length $n$ in which each element (except the last $a_n$) has one outgoing $r$-edge to a "dead end" and one outgoing $r$-edge to an $r$-self-loop. It is readily verified that there is no $\Lmc(\{\exists,\forall,\sqcup\})$ concept $D$ such that $a_i\in (\forall r.D)^\Imc$, for $1\leq i< n$. We will rely on this fact below. 

To construct the negative examples, let us first define a set of \EL concepts $S_n$, inductively as follows:
\begin{align*}
S_0 & = \{\top\} \\
S_{i+1} & = \{\exists r.(\alpha\sqcap C)\mid C\in S_i,\alpha\in \{A,B\}\}, \text{ for $i\geq 0$}
\end{align*}
Clearly, $S_n$ contains $2^n$ \EL concepts, and $C_T\sqsubseteq C$ and $\lVert C \rVert \leq 3n$, for every $C\in S_n$. Consider the example $(\Imc_C,a_C)$ with $C\in S_n$. Note that $\Imc_C$ contains a single assertion that mentions $a_C$. By Theorem~\ref{thm:duals}, $(\Imc_C,a_C)$ has a singleton $\Sigma$-simulation dual $\{(\Imc_C'',a_C')\}$ with $\lVert \Imc_C'' \rVert \leq 3\cdot|\Sigma|\cdot 3n$. These duals are negative examples for $C_T$ since $C_T\sqsubseteq C$, using Lemma~\ref{lem:subsumption-simulation} and the definition of duals. Clearly, $(\Imc,a_0)$ is a positive example for $C_T$. We slightly modify $\Imc_C''$ by adding a fresh $r$-successor to every element, calling the resulting interpretation $\Imc'_C$. Note that $(\Imc_C',a_C')$ and $(\Imc_C'',a_C')$ are simulation-equivalent, due to totality of $r^{\Imc_C''}$. So they satisfy precisely the same $\EL$ concepts, and hence all $(\Imc_C',a_C')$  are negative examples.

Let the probability distribution $\mathbb P$ assign probability $1/2$ to (positive example) $(\Imc,a_0)$ and probability $1/{2^{n+1}}$ to each $(\Imc_C',a_C')$, $C\in S_n$.

\medskip
\noindent\textit{Claim~1.}
For each $C\in S_n$, $C$ is the most general $\EL$ concept that fits the negative example $(\Imc_C',a'_C)$.

\smallskip Claim~1 is from the original proof, so we do not prove it here. Instead, we prove another claim that allows us to reuse Claim~1. Its proof relies on the positive example. 

\medskip
\noindent\textit{Claim~2.} 
Let $C\in S_n$ and consider $P=\{(\Imc,a_0)\}$ and $N=\{(\Imc_C',a_C')\}$.
If $D\in\Lmc(\{\exists,\forall,\sqcap\})$ fits $P,N$, then $D\sqsubseteq C$.

\medskip\noindent\textit{Proof of Claim~2.} We inductively show that, for $\ell=0 \ldots n$: 
\begin{itemize}
    \item[$(\ast)$]  if there is some $\Lmc(\{\exists,\forall,\sqcap\})$
    concept $D$ that distinguishes $(\Imc,a_\ell)$ from some $(\Imc_C',d)$, $d\in
    \Delta^{\Imc_C''}$, then there is some \EL concept $\widehat D$ with
    $D\sqsubseteq \widehat D$ that distinguishes $(\Imc,a_\ell)$ from
    $(\Imc_C',d)$.

\end{itemize}

\smallskip
For the inductive start, consider $a_n$. By construction, $a_n$ does not satisfy any $\Lmc(\{\exists,\forall,\sqcap\})$ concept of shape $\forall r.E$. Let now have $D$ a top-level conjunct of shape $\exists r.E$. Clearly, $E$ cannot have top-level conjunct $A/B$. Moreover, $E$ cannot be of shape $\forall r.C$ since any such $E$ is satisfied in $d$, due to the fresh $r$-successors (added when transitioning from $\Imc''_C$ to $\Imc_C'$). Hence, $E$ has top-level conjunct $\exists r.F$, then only $c_n$ can be the witness of $E$. But the only $\Lmc(\{\exists,\forall,\sqcap\})$ concepts satisfied by $c_n$ are equivalent to $\exists r^u.\top$ for some $u \geq 0$. But these are satisfied by $d$ as well, so do not help in distinguishing. It follows that $D$ has a top-level conjunct $A/B$ and this is the distinguishing \EL concept that we claimed to exist. 

\smallskip For the inductive step, consider some $a_\ell$, and let $D\in \Lmc(\{\exists,\forall,\sqcap\})$ be the concept distinguishing $a_\ell$ from $d$ of the form 
\[D=A_1\sqcap \ldots \sqcap A_\ell\sqcap \forall r.D'\sqcap \exists r.D_1\sqcap\ldots\sqcap \exists r.D_k.\]
We analyze the shape of $D$. Clearly, by construction of $\Imc$ and since $(\Imc,a_\ell)$ is a positive example for $D$, we crucially have $D'=\top$ as $a_\ell$ does not satisfy any non-trivial concept of shape $\forall r.E$. Hence, there is some $D_i$ that is satisfied in a successor of $a_\ell$, but in none of the successors of $d$. We distinguish cases:
\begin{itemize}
    \item $D_i$ cannot be of shape $D_i=\forall r.D''$, since any such concept is satisfied at the fresh $r$-successor of $d$ (added when transitioning from $\Imc''_C$ to $\Imc_C'$).
    
    \item If $D_i$ contains some top-level conjunct $A/B$, then the mentioned successor of $a_\ell$ satisfying $D_i$, is $a_{\ell+1}$ and we can take $\widehat D=\exists r. \widehat {D_i}$ where $\widehat {D_i}$ is the \EL concept obtained from the induction hypothesis applied to $a_{\ell+1}$ and any $r$-successor of $d$.
   
    \item If $D_i$ contains no top-level conjunct of shape $A/B$, but a top-level $\exists r.C$, then the mentioned successor of $a_\ell$ cannot be $b_\ell$, since $b_\ell$ does not have any $r$-successor. It can also not be $c_\ell$ since the only $\Lmc(\{\exists,\forall,\sqcap\})$ concepts that are satisfied in $c_\ell$ are equivalent to $\exists r^u.\top$ for some $u \geq 0$ and each such concept is satisfied in some successor of $d$, due to totality of $r^{\Imc''_C}$. Hence, the mentioned successor of $a_\ell$ satisfying $D_i$, is $a_{\ell+1}$, and it remains to apply induction as in the previous case. 

\end{itemize}

This finishes the proof of Statement~$(\ast)$. To see the statement from the claim, let $D\in\Lmc(\{\exists,\forall,\sqcap\})$ fit $P,N$. By~$(\ast)$, there is an \EL concept $\widehat D$ that fits $P,N$ and such that $D\sqsubseteq \widehat D$. By Claim~1, $\widehat D\sqsubseteq C$, hence also  $D\sqsubseteq C$ as required.
\hfill$\dashv$

\medskip Let now $P,N$ be a sample of $m(1/\delta,1/\varepsilon,3,p(n),3n)$ examples sampled from $\mathbb P$. With probability at least $1-\delta=1/2$, $P$ contains the (only) positive example $(\Imc,a_0)$. 

\medskip\noindent\textit{Claim 3.} $C_H=\bigsqcap_{(\Imc_C'',a_C')\in N} C$ is the most general $\Lmc(\{\exists,\forall,\sqcap\})$ concept fitting $P,N$.

\medskip\noindent\textit{Proof of Claim~3.} We first verify that $C_H=\bigsqcap_{(\Imc_C'',a_C')\in N} C$ fits $P,N$. By construction of \Imc, $(\Imc,a_0)$ is a positive example for each $C\in S_n$. Moreover, by the definition of duals, $(\Imc_C'',a_C')$ is a negative example for each such $C$.

We now verify that it is the most general one. Let $D$ be any $\Lmc(\{\exists,\forall,\sqcap\})$ concept that fits $P,N$. By Claim~1 and Claim~2 together, $D\sqsubseteq C$, for each $(\Imc_{C}'',a_C')\in N$. Hence, $D\sqsubseteq C_H$.
\hfill$\dashv$

\medskip Hence, (a concept equivalent to) $C_H$ is output by the algorithm. We claim that $C_H$ labels all negative examples not in the sample wrong. Indeed, $C_H\not\sqsubseteq C'$ for any $C'$ with $(\Imc''_{C'},a_C')\notin N$ (not difficult to see and shown in the original proof). By definition of duals, $(\Imc_{C'}'',a_C')$ is labeled positively by~$C_H$. Since there are $2^{n}$ negative examples, each with probability $1/{2^{n+1}}$, the choice of $n$ implies that the error is greater than $\varepsilon=1/2$.
\qed 
\end{proof}

\lemboundedfittingexistsor*

\begin{proof}
We prove the statement for $\Lmc(\{\exists,\sqcup\})$, the other part is dual, due to Lemma~\ref{lem:fitting_dual}. Let $P,N$ be sets of positive and negative examples and let $C$ be an $\Lmc(\{\exists,\sqcup\})$ concept of minimal size fitting $P,N$. Such concept is returned by bounded fitting. Due to minimality of $C$, we can assume that
\begin{itemize}
    \item[$(\ast)$] $C$ does not have a subconcept of shape $\exists r.C_1\sqcup \exists r.C_2$,
\end{itemize}
as it could be equivalently replaced by the smaller concept $\exists r.(C_1\sqcup C_2)$.

Let $D$ be any concept at least as specific than $C$, that is, $D\sqsubseteq C$. Without loss of generality, we can assume that also $D$ satisfies~$(\ast)$. Consider the syntax trees of $C$ and $D$. A \emph{path in a syntax tree} is any sequence $r_1\ldots r_kA$ that can be read by starting at the root of the syntax tree, reading only the role name for $\exists r$-nodes, skipping $\sqcup$-nodes, and finally reading the label $A/\top$ of the leaf. 

It can be verified that $D\sqsubseteq C$ holds if and only if all paths in the syntax tree of $D$ are paths in the syntax tree of $C$. As $D\sqsubseteq C$ by assumption, all paths in $D$ are paths in $C$. Based on the fact that $C$ and $D$ satisfy Property~$(\ast)$ above, one can then
show that there is an injective homomorphism (which is just a mapping that preserves node labels) from the syntax tree of $D$ into the syntax tree of $C$. But this means that the syntax tree of $D$ is at most as large as the syntax tree of $C$. Equality is only the case if $C\equiv D$.

Hence, if $D$ would also fit $P,N$, then bounded fitting would have returned $D$ instead of $C$, a contradiction.\qed

\end{proof}

\end{document}